%% file: main.tex
\documentclass{article}

\PassOptionsToPackage{numbers, compress}{natbib}

 \usepackage[preprint]{neurips_2026}

\usepackage[utf8]{inputenc} 
\usepackage[T1]{fontenc}    
\usepackage{hyperref}       
\usepackage{url}            
\usepackage{booktabs}       
\usepackage{amsfonts}       
\usepackage{nicefrac}       
\usepackage{microtype}      
\usepackage{xcolor}         
\usepackage[table]{xcolor}

\usepackage{amsmath}
\usepackage{amssymb}
\usepackage{mathtools}
\usepackage{amsthm}
\usepackage{multirow}
\usepackage{wrapfig}

\theoremstyle{plain}
\newtheorem{theorem}{Theorem}[section]
\newtheorem{proposition}[theorem]{Proposition}
\newtheorem{lemma}[theorem]{Lemma}
\newtheorem{corollary}[theorem]{Corollary}
\theoremstyle{definition}
\newtheorem{definition}[theorem]{Definition}

\theoremstyle{remark}
\newtheorem{remark}[theorem]{Remark}

\newcommand{\fr}[1]{\mathfrak{#1}}

\newcommand{\mcal}[1]{\mathcal{#1}}

\newcommand{\mbf}[1]{\mathbf{#1}}
\newcommand{\hbf}[1]{\hat{\mathbf{#1}}}
\newcommand{\bb}[1]{\mathbb{#1}}

\newcommand{\G}[2]{\text{#1}(#2)}

\title{Why Depth Matters in Parallelizable Sequence Models: A Lie Algebraic View}

\author{%
  Gyuryang Heo$^{1,2}$
  \hspace{1em} Timothy Ngotiaoco$^{2}$
  \hspace{1em} Kazuki Irie$^{3}$\\
  \hspace{1em} \textbf{Samuel J.~Gershman}$^{2,3}$
  \hspace{1em} \textbf{Bernardo L.~Sabatini}$^{1,2}$\\
  $^{1}$ Howard Hughes Medical Institute,\\
  Department of Neurobiology, Harvard Medical School, Boston, MA, USA\\
  $^{2}$ Kempner Institute for the Study of Natural and Artificial Intelligence,\\
  Harvard University, Cambridge, MA, USA\\
  $^{3}$ Department of Psychology and Center for Brain Science,\\
  Harvard University, Cambridge, MA, USA\\
  \texttt{gheo@g.harvard.edu, timothy\_ngotiaoco@harvard.edu,}\\
  \texttt{\{kirie, gershman\}@fas.harvard.edu, bernardo\_sabatini@hms.harvard.edu}
}

\begin{document}

\maketitle

\begin{abstract}
Scalable sequence models, such as Transformer variants and structured state-space models,
often trade expressivity power for sequence-level parallelism, which enables efficient training.
Here we examine the bounds on error and how error scales when models operate outside of their expressivity regimes using a Lie-algebraic control perspective.
Our theory formulates a correspondence between the depth of a sequence model and the tower of Lie algebra extensions.
Echoing recent theoretical studies, we characterize the Lie-algebraic class of constant-depth sequence models
and their corresponding expressivity bounds.
Furthermore, we analytically derive an approximation error bound
and show that error diminishes exponentially as the depth increases,
consistent with the strong empirical performance of these models.
We validate our theoretical predictions using experiments on symbolic word and continuous-valued state-tracking problems.\looseness=-1
\end{abstract}

\section{Introduction}
Most sequence problems, from symbolic processing tasks like natural language, mathematics, and coding, to the understanding of physical-world dynamics, are fundamentally \emph{order-sensitive}.
Yet, highly scalable sequence models often achieve parallelism via \emph{order symmetry}.
Notably,
self-attention layers in transformers \citep{trafo}, 
and crucial components of diagonal structured state-space models \citep{orvieto2023resurrecting, gu2023mamba} or linear transformer variants \citep{katharopoulos2020transformers,schmidhuber1992learning, YangWSPK24} 
exhibit order symmetry in their layer inputs.
This strong structural bias has inspired the study of the expressivity limits of such models \citep{hahn2020theoretical,MerrillWGSSY20}.
In particular, recent theoretical work has identified reasoning and state-tracking problems \citep{KimS23,merrill2024illusion,grazzi2024unlocking}
that are provably unsolvable by constant-depth transformers or diagonal structured state-space models.

Despite these theoretical bounds, deep scalable sequence models remain the most empirically successful architectures,
especially in large-scale language modeling.
This discrepancy motivates a quantitative question:
\emph{How badly does a model operate when applied to tasks it provably cannot solve exactly?}
Quantifying this gap matters
because more efficient models
would be preferable
if the error remains tolerable.
Although prior work \citep{muca2024theoretical, walker2025structured} provides related intuition,
the error-expressivity scaling laws have not, to our knowledge, been clearly derived.

To answer this question, we leverage Lie theory 
\citep{iserles2000lie,agrachev2013control,robbin2022introduction}.
Geometrically, Lie theory measures the order-sensitivity of operations:
the state discrepancy caused by swapping the order of events (see Fig.~\ref{fig:Informal_Lie} for an intuitive illustration).
This provides both a \emph{hierarchy} for classifying tasks and sequence models
under their \emph{degree} of order sensitivity,
and a method for quantifying approximation error as the \emph{mismatch} between target and model states under the same input.

We show that order-symmetric models incur an unavoidable approximation error
when the task is order sensitive.
We then identify depth as a mechanism that
extends the expressivity of parallelizable sequence models, decreasing this error.
Overall, our work characterizes depth-dependent error-expressivity behavior, and provides a structural lens for model choice based on the task structure.
\section{Mathematical Preliminaries}
\label{sec:math_prelim}
Here we introduce the mathematical framework used throughout the paper.
We formalize order sensitivity and
approximation error
to establish our theories in Sec.~\ref{sec:theory}.
Readers primarily interested in the main results and their practical implications may wish to proceed directly to Sec.~\ref{sec:theory} and Sec.~\ref{sec:exp}.

\subsection{Lie groups and Lie algebras}
\label{sec:LieAlgebra}
A \emph{Lie group} $G$ is a smooth manifold equipped with a group structure under smooth multiplication.
Its tangent space at the identity characterizes a \emph{Lie algebra} $\fr{g}$.
$\fr{g}$ is a vector space equipped with a bilinear and antisymmetric Lie bracket operation $[\cdot, \cdot]: \fr{g} \times \fr{g} \to \fr{g}$ satisfying the Jacobi identity:
\begin{equation*}
    [[X,Y],Z] = [X,[Y,Z]] - [Y,[X,Z]],
    \quad X,Y,Z \in \fr{g}.
\end{equation*}
A canonical example of the Lie bracket is a matrix commutator.
For matrices
$\mbf{A}, \mbf{B} \in \bb{R}^{n \times n}$,
the Lie bracket corresponds to the matrix commutator $[\mbf{A},\mbf{B}] = \mbf{A}\mbf{B} - \mbf{B}\mbf{A}$.

\begin{wrapfigure}{r}{0.4\textwidth}
  \begin{center}
\includegraphics[width=0.25\columnwidth]{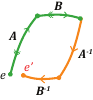}
  \end{center}
  \caption{Geometric intuition of Lie theory. Initialized at $e$,
  consider \emph{actions} $\mbf{A}$ and $\mbf{B}$ sequentially,
  and \emph{undo} actions $\mbf{B}^{-1}$ then $\mbf{A}^{-1}$.
  Composition of actions $\mbf{A}\mbf{B}\mbf{B}^{-1}\mbf{A}^{-1}$ would return to $e$.
  However,  \emph{switching} the order of undoing,
 by using $\mbf{A}\mbf{B}\mbf{A}^{-1}\mbf{B}^{-1}$, might not return to $e$,
incurring a discrepancy by landing on point $e'$.
  Lie theory provides a measure of the potential offset between $e$ and $e'$.}
  \label{fig:Informal_Lie}
\end{wrapfigure}

Given a vector space spanned by a collection of vector fields $\mcal{F}:= \{F_i\}$, the smallest Lie algebra is denoted as $\text{Lie}(\mcal{F})$ \cite{iserles2000lie}.
Unless otherwise mentioned, Lie algebras are finite-dimensional over $\bb{R}$.

\paragraph{Classes of Lie algebras}
Here we classify Lie algebras based on their \emph{degree} of order sensitivity.
The \emph{derived} series of $\fr{g}$ is the inductive sequence of \emph{commutator ideals}:
\begin{equation*}
    \fr{g}^{(0)} = \fr{g},
    \quad
    \fr{g}^{(i)} = [\fr{g}^{(i-1)},\fr{g}^{(i-1)}].
\end{equation*}
If $\fr{g}^{(j)} = 0$ with the smallest finite number $j$, $\fr{g}$ is classified as a \emph{solvable} Lie algebra with derived length $j$.
Similarly, the \emph{lower central} series of $\fr{g}$ is a sequence of ideals spanned by nested commutators:
\begin{equation*}
    \fr{g}^0 = \fr{g},
    \quad
    \fr{g}^i = [\fr{g},\fr{g}^{i-1}].
\end{equation*}
If $\fr{g}^c = 0$ with the smallest finite number $c$, $\fr{g}$ is \emph{nilpotent} with class $c$.
An \emph{abelian} Lie algebra is the special case with $[\fr{g},\fr{g}]=0$ \cite{hall2013lie}.

These Lie algebras are related by a structural hierarchy.
Abelian Lie algebras are nilpotent, and all nilpotent Lie algebras are solvable.
Furthermore, one can \emph{glue} abelian Lie algebras to construct a solvable one.
A Lie algebra \emph{extension} is a short exact sequence $0 \to \fr{a} \to \fr{g} \to \fr{h} \to 0$ where $\fr{a}$ is an \emph{ideal} of $\fr{g}$ and $\fr{h} \cong \fr{g}/\fr{a}$.
If $\fr{g}$ has derived length $j$, $\fr{g}^{(j-1)}$ is an abelian ideal by definition, and $\fr{g}/\fr{g}^{(j-1)}$ is again solvable with derived length $j-1$.
Thus, a solvable Lie algebra is a \emph{tower} of abelian algebra extensions (App.\ref{app:subsubsec_extensionprelim}).
\subsection{State-space models and the Lie equation}
\label{sec:LPVsystem}
To connect Lie theory with sequence models, we formalize state-space models
as controlled dynamical systems on Euclidean space (App.\ref{app:control_detail}).
We will observe how the Lie algebra characterizes a system.
\begin{definition}[State space model]
\label{def:staticLPVsystem}
A \emph{state space model} (SSM) is an affine vector field
on Euclidean \emph{state space} $\mcal{M} = \bb{R}^n$
and \emph{input space} $\mcal{X}$
such that:
\begin{equation}
\label{eq:vectorstateODE}
F_x: \mcal{M} \to T\mcal{M},
\quad F_x(h) := \mbf{A}(x) h + \mbf{b}(x).
\end{equation}
$\mbf{A}: \mcal{X} \to \bb{R}^{n \times n}$ and $\mbf{b}: \mcal{X} \to \bb{R}^n$ are real analytic \cite{sontag2013mathematical}.
We call $\mbf{A}$ the \emph{generator} of the SSM.
We denote the SSM as $\mbf{S}$, and $\dim \mbf{S} := \dim \mcal{M}$.
\end{definition}
We assume standard control-theoretic regularity and minimality (App.~\ref{app:control_detail}).
With a locally bounded and integrable input path $x(\cdot): [0,T] \to \mcal{X}$,
one can reintroduce time in Eq.~\ref{eq:vectorstateODE} and construct an ODE.
$\mbf{S}$ is \emph{initialized} when the initial state $h(0)$ is fixed.
In what follows, we focus on the state evolution.
\begin{definition}[State-transition matrix]
\label{def:Phi}
    The state-transition matrix $\Phi$
    is an evolution operator of the homogeneous flow:
    \begin{equation}
        \label{eq:STM}
        \Phi_x(t,0) = \mcal{T} \exp \int_0^t \mathbf{A}(x(\tau)) \; d\tau,
    \end{equation}
    where $\mcal{T} \exp$ denotes the chronological (time-ordered) exponential \cite{agrachev2013control, sontag2013mathematical}.
\end{definition}
\paragraph{From state-centric to flow-centric view}
\label{sec:flow_centric_realization}
Rather than tracking a single state, one can consider tracking entire \emph{flows}.
This perspective bridges Lie algebras and SSMs.

Let us note some key features of $\Phi(t,0)$:
\begin{equation}
\label{eq:PhiFeature}
    \dot{\Phi}(t,0) = \mbf{A}(x(t))\Phi(t,0),
    \quad
    \Phi(0,0) = I.
\end{equation}
It turns out that this is a \emph{Lie equation} \cite{iserles2008magnus}:
\begin{definition}[Controlled Lie equation]
\label{def:LieGroupEquation}
A linear controlled Lie equation is a matrix ODE:
\begin{equation}
\label{eq:LieGroupEquation}
    \dot{\Psi}(t) := \mbf{A}(x(t)) \Psi(t),
    \quad
    \Psi(0) = I,
\end{equation}
where $\text{Lie}(\{\mbf{A}(x)\}) = \fr{g} \subseteq \fr{gl}(n,\bb{R})$, $\Psi \in G \subseteq \G{GL}{n,\bb{R}}$, and the Lie group $G$ locally integrates $\fr{g}$.
\end{definition}
It is immediately clear that Eq.~\ref{eq:PhiFeature} aligns with Def.~\ref{def:LieGroupEquation}:
$\Phi$ can be considered as a local curve on $G$.
Therefore, the controlled linear Lie equation is a \emph{flow-centric} ODE,
tracking all sets of local evolutions rather than a single state \cite{jurdjevic1972control}.

As $\fr{g}$ locally characterizes state evolution,
we classify SSMs with the algebraic class of $\fr{g}$
and, for short, denote it $\mbf{S}_\fr{g}$.
We map SSMs to its Lie equations via \emph{lifting}.
\begin{definition}[Lift]
    Consider an $\mbf{S}_\fr{g}$ with initial state $h_0$.
    Its lifted Lie equation extends Eq.~\ref{eq:LieGroupEquation}
    with augmented flow space $G \times \mcal{M}$.
    For $\Psi \in G, \: h_0 \in \mcal{M}$:
    \begin{equation*}
        \dot{\Psi}(t) = \mbf{A}(x(t))\Psi(t),
        \
        \dot{h}(t) = 0,
        \
        \bigl( \Psi(0), h(0) \bigr) = \bigl( I, h_0 \bigr).
    \end{equation*}
\end{definition}
\paragraph{Restricted SSM}
\label{sec:main_affine_vector_field}
As an umbrella term, we say $\mbf{S}_\fr{g}$ is \emph{restricted}
if $\text{Lie}(\{\mbf{A}(x)\})$ is abelian.
Although
restricted SSMs can be order-sensitive due to the translation term, $\mbf{b}(x)$,
we observe that SSM expressivity is mostly determined by the generator, $\mbf{A}(x)$ (Thm.~\ref{thm:sim_error}).\footnote{
One can absorb $\mbf{b}$ via \emph{homogenization} and treat an affine SSM
as a homogeneous SSM (App.~\ref{app:bridge_to_general_LPV}; \cite{krener1977decomposition}).}
If not restricted, $\mbf{S}_\fr{g}$ is \emph{general}.
\subsection{Simulation and approximation error}
\label{par:sim_min_main}
Studying expressivity of a model class is tied to \emph{simulation}:
whether one system can reproduce another's behavior \cite{liu2022transformers}.
Analogously, we use a \emph{state-based} notion of simulation.

We say $\mbf{S}_0$ simulates $\mbf{S}_1$
if a smooth surjective submersion $\varphi: \mcal{M}_0 \to \mcal{M}_1$ maps
$\mbf{S}_0$ state evolution to $\mbf{S}_1$ states under any inputs.
When exact simulation is impossible, we quantify the state mismatch after evolution.\footnote{
Details are provided in App.~\ref{app:sim_detail}.}
Note that the lifted Lie equation simulates the original initialized SSM,
\begin{equation}
\label{eq:flow_to_state_submersion}
    \varphi: G \times \mcal{M} \to \mcal{M},
    \quad
    \varphi\bigl( \Phi, h \bigr) = \Phi h,
    \quad
    h(t) = \varphi\bigl( \Phi(t,0), h(0) \bigr) = \Phi(t,0)h(0).
\end{equation}
\paragraph{Magnus expansion}
\label{par:MagnusExpansion}
Given the Lie-algebraic perspective of SSMs,
we now introduce a quantitative tool for the approximation error.
The \emph{Magnus expansion} decomposes $\Phi$ via exponential of
iterated Lie brackets (App.~\ref{app:CDE}),
providing an explicit measure of order-relevant errors.
\begin{definition}[Commutator mass]
\label{def:commutator_mass}
    Given a fixed input path $x$ over the interval $[0,T]$,
    we define \emph{commutator mass} of $\mbf{S}$
    as the norm of the second-order Magnus term, $\|\Omega_2\|$:
    \begin{equation}
    \label{eq:omega_2}
        \Omega_2
        = \frac{1}{4}\iint_{[0,T]^2}
        \text{sgn}(t_1 - t_2)
        [\mathbf{A}(x(t_1)), \mathbf{A}(x(t_2))]
        \, dt_1 dt_2.
    \end{equation}
\end{definition}
We work on small windows in which the Magnus series converges.
Long horizon behavior for large $T$ is constructed by the composition of small time partitions \cite{sontag2013mathematical}:
\begin{equation}
\label{eq:flow_composition}
    \Phi(T,0) = \prod_{i=0}^n \Phi(t_{i+1}, t_i),
    \quad
    t_0=0,
    \
    t_{n+1}=T.
\end{equation}
\subsection{Deep structure}
Most practical applications of artificial neural networks leverage deep layered structures with additional pointwise computation (e.g. MLP) between layers. To extend our theory, here we discuss dynamical systems with deep layered structure  \cite{krener1977decomposition}.
\begin{definition}[Deep SSM]
\label{def:stacking}
    The \emph{stack} of $\mbf{S}_0$ and $\mbf{S}_1$ is
    defined with a smooth map $v: \mcal{M}_0 \times \mcal{X} \to \mcal{X}_1$:
    \begin{align*}
        \mcal{M} &\cong \mcal{M}_0 \times \mcal{M}_1, \: x_1(t) = v(h_0(t), x(t)), \\
        \dot{h}_0(t) &= \mbf{A}_0(x(t))h_0(t) + \mbf{b}_0(x(t)), \quad
        \dot{h}_1(t) = \mbf{A}_1(x_1(t))h_1(t) + \mbf{b}_1(x_1(t)).
    \end{align*}
     A deep SSM is obtained inductively:
     \begin{align*}
        x_i(t) &= v_i(h_{i-1}(t), \cdots, h_0(t), x(t)),
        \quad
        x_0(t) = x(t),
        \quad
        \dot{h}_i(t) := \mbf{A}_i(x_i(t)) h_i(t) + \mbf{b}_i(x_i(t)),
     \end{align*}
     where $v_i: \Bigl( \prod_{j=0}^{i-1} \mcal{M}_j \Bigr) \times \mcal{X} \to \mcal{X}_i$ is smooth.

     We refer to each local SSM as the $i$-th layer.
     An $L$-layer SSM is said to be restricted when all of its layers are restricted.
     This convention applies to the other classes as well.
\end{definition}
\subsection{Word problems}
\label{sec:intro_word_problems}
Lastly, we introduce \emph{word problems}, one of the standard test beds for sequence model expressivity.
\begin{definition}[Word problem \citep{merrill2024illusion, liu2022transformers}]
\label{def:wp_main_text}
    Given a finite alphabet set $A$, a \emph{word} is an element of the free monoid $A^\ast$ on $A$.
    For a finite monoid $M$ generated by $A$, the word evaluation map $f$ is the unique monoid homomorphism
    $f: A^\ast \to M$ determined by sending each generator in $A$ to itself in $M$.
    Solving the word problem for $M$ is to learn the word evaluation map and reduce any word
    $a_1a_2\ldots a_t$ to $a_1 \cdot a_2 \cdot \ldots \cdot a_t \in M$ under multiplication $\cdot$.
\end{definition}
Consider each \emph{alphabet symbol} as a \emph{primitive action}, such as a permutation, rotation, or state transition operation (Fig.~\ref{fig:Informal_Lie}).
A \emph{word} is then a sequence of \emph{action commands} evaluated by composing primitive actions in order.
A sequence of matrix products and a rollout of action-conditioned state transitions in a Markov decision process (MDP) are examples of word evaluation.
The \emph{parity} problem is the word problem for the cyclic group $C_2$, with $\cdot$ corresponding to addition modulo $2$.

Thus, word problems are symbolic state-tracking tasks with an explicit algebraic class.
As symbols arrive, the model must update its state according to the composition rule
and output the resulting element.
We list classes of exemplary word problems and related concrete problems in Table~\ref{tab:words}.
\begin{table}[t]
\setlength{\tabcolsep}{0.3em}
  \caption{Class of representative word problems on finite group.}
  \vspace{-3mm}
  \label{tab:words}
  \begin{center}
    \begin{small}
        \begin{tabular}{lcccr}
          \toprule
          Class  & Finite group example         & Related problems\\
          \midrule
          Abelian    & $C_2, C_3, C_{60}$ & Parity; Cyclic gridworld \\ 
         Nilpotent & $D_8, H_3$ & Quantum spin \\
          Solvable    & $S_3, A_4, S_4$ & 2D affine transformation \\
          Non-solvable    & $A_5, S_5$ & Rubik's cube; 3D rotation\\
          \bottomrule
        \end{tabular}
    \end{small}
  \end{center}
  \vskip -0.1in
\end{table}
\section{Theory}
\label{sec:theory}
Our theory has two major parts.
Thm.~\ref{thm:sim_error} derives an approximation error bound for $1$-layer SSMs.
Thm.~\ref{thm:depth_extension} and Cor.~\ref{col:Kstacks} provide theoretical ground on how depth extends expressivity via Lie algebra extension.
Cor.~\ref{col:nilpotentization} and Prop.~\ref{prop:logdepth}
construct an upper bound for depth-dependent approximation errors.
\subsection{Expressivity bounds of a single layer}
\label{sec:theory_single_layer}
Here our theory aligns with prior work \cite{merrill2024illusion, liu2022transformers, anonymous2025the}.
We begin with the most basic case.
\begin{lemma}
\label{lem:sim_impossible}
No abelian SSM $\hbf{S}$ can simulate a general SSM $\mbf{S}$.
\end{lemma}
Proof is provided in App.~\ref{app:sim_impossible_proof}.
Building on Lemma~\ref{lem:sim_impossible}, we constructively provide an error bound.
\begin{theorem}
\label{thm:sim_error}
    Consider a restricted $\hbf{S}$ that approximates $\mbf{S}$.
    There exists an input path $x$
    on a sufficiently small $[0,T]$
    that incurs approximation error scaling with the commutator mass $\|\Omega_2\|$ of $\mbf{S}$.
\end{theorem}
Proof is provided in App.~\ref{app:sim_error_proof}.
\begin{proof}[Proof sketch]
The order symmetry generates an irreducible
order-sensitive error.
Via flow composition, the error can accumulate over long horizons and scale with sequence length
(Eq.~\ref{eq:flow_composition}).
In Fig.~\ref{fig:Informal_Lie}, this corresponds to repeatedly composing the action and accumulating the gap between $e$ and $e'$.
\end{proof}
\subsection{Expressivity bounds 
of deep structure}
\label{sec:theory_multi_layer}
Here we consider deep restricted SSMs
and discuss how depth mitigates their expressivity limits.
\paragraph{Exact representation}
We begin by clarifying an algebraic class of deep restricted SSMs.
\begin{proposition}
\label{lem:stackderivedlength}
    Restricted $k$-layer SSMs are solvable with at most $2k$ derived length.
    Abelian $k$-layer SSMs are
    solvable with at most $k$ derived length.
\end{proposition}
Proof is provided in App.~\ref{app:stackderivedlength_proof}.
Thus,
finite-depth restricted SSMs are limited to be solvable.

We now arrive at our core theorem, connecting the depth of SSM to Lie algebra extension. 
\begin{theorem}
\label{thm:depth_extension}
Let $0 \to \fr{a} \to \fr{g} \to \fr{h} \to 0$ be a Lie algebra extension.
Then, there exists a 2-layer deep SSM with smooth output map that simulates $\mbf{S}_\fr{g}$.
\end{theorem}
Proof is provided in App.~\ref{app:sub_proof_depth_extension}.
An immediate corollary for the order-symmetric structure follows.
\begin{corollary}
\label{col:Kstacks}
    Consider $\mbf{S}_\fr{g}$ such that the derived length of $\fr{g}$ is $k$.
    Then, there exists an abelian $k$-layer SSM plus smooth output map
   that simulates $\mbf{S}_\fr{g}$.
\end{corollary}
Proof is provided in App.~\ref{app:sub_proofKstacks}.
Cor.~\ref{col:Kstacks} shows that deep abelian SSMs are quite powerful as the deep structure permits abelian layers to generate order-sensitive solvable flows.

We again use MDPs to build tangible intuition. As noted in Table~\ref{tab:words}, different orders of action execution generally induce different action-dependent transition matrices.
    Except in simple cases (e.g. cyclic gridworld), modeling order-sensitive dynamics is critical for a sequence model to represent an MDP or serve as a learned \emph{world model} in model-based reinforcement learning \cite{chen2022transdreamer}.
    Thm.~\ref{thm:depth_extension} and Cor.~\ref{col:Kstacks} characterize how many layers are needed to represent \emph{compositions} of actions with the correct transition structure.
    In this view, each layer tracks a collection of \emph{commuting} components of dynamics, whose composition reconstructs non-commuting transitions.
\paragraph{Approximation}
We now turn our attention to approximation.
Cor.~\ref{col:Kstacks} naturally maps to the well-known control theory literature.
Given an input path on a short interval $[0,T]$,
we define the local \emph{generator mass} of $\mbf{S}$:
    \begin{equation*}
        \epsilon = \int_0^T \|\mbf{A}(x(t))\| \, dt < 1
        \quad \text{(for sufficiently small $T$)}.
    \end{equation*}
\begin{corollary}
\label{col:nilpotentization}
    For non-solvable $\mbf{S}_\fr{g}$, there exists an abelian $k$-layer SSM whose
    approximation error scales with $\mcal{O}(\epsilon^{2^k})$.
\end{corollary}
Proof is provided in App.~\ref{app:sub_proof_nilpotentization}.
Note that Thm.~\ref{thm:sim_error}
corresponds to $k=1$. $\|\Omega_2\|$ scales in $\mcal{O}(\epsilon^2)$.

Cor.~\ref{col:nilpotentization} is a promising result: although algebraically obstructed,
deep structure exponentially mitigates the order-sensitive error.
The local error can again extend to long horizons via Eq.~\ref{eq:flow_composition}.

For word problems, we treat each input symbol as a piecewise-constant input. Under piecewise-constant inputs, the SSMs in our theory can be considered as discrete-time sequence models.
\begin{proposition}
\label{prop:logdepth}
    Any word problem with word length bounded by $T$
    can be simulated by an
    abelian deep SSM with at most $\bigl(\lfloor \log_2T \rfloor + 1\bigr)$-layers and a smooth output map.\footnote{
Similar logarithmic depth result for transformers was reported in \citet{liu2022transformers},
and empirically observed later \cite{hu2024limitation}.
}
\end{proposition}
Proof is provided in App.~\ref{app:sub_proof_logdepth}.
Prop.~\ref{prop:logdepth} provides an explicit upper bound on the depth required to simulate discrete, finite-length sequences.
Note that this is the worst-case depth upper bound:
much shallower SSMs may successfully simulate a well-structured word problem.
For example, a 1-layer abelian SSM suffices to simulate an abelian word problem for any $T$.

Lastly, we note that solving the bounded length word problem can come at a cost of state expansion.
\begin{corollary}
\label{col:logdepth_state}
Consider
an alphabet of size
$|\mcal{A}| = n$ and word length bounded by $T$.
Then, the total state space dimension of the abelian deep SSM constructed in Prop.~\ref{prop:logdepth} exponentially grows in $T$ for fixed $n$ with asymptotic order $\mcal{O}(n^T/T)$.
\end{corollary}
Proof is provided in App.~\ref{app:sub_proof_logdepth_state}. Again, this is a worst-case width upper bound following Prop.~\ref{prop:logdepth}.
Together, Prop.~\ref{prop:logdepth} and Cor.~\ref{col:logdepth_state} clarify how depth and width play roles in constant length simulation.
\begin{remark}[Conjecture on restricted SSMs]
\label{rem:conjecture_restricted_limits}
    Thm.~\ref{thm:sim_error} suggests that restricted SSMs do not substantially extend expressivity beyond abelian SSMs. Although restricted SSMs can be order-sensitive (Sec.~\ref{sec:main_affine_vector_field}), their structural constraints prevent them from simulating general SSMs.
    We conjecture that there exists a $\mbf{S}_\fr{g}$ of derived length $k+1$ that cannot be simulated by any restricted $k$-layer SSM.
\end{remark}

\section{Experiments}
\label{sec:exp}
Here we use experiments as diagnostic tests of the theory.
We examine symbolic word problems (WPs; Sec.~\ref{sec:exp_wp}) covering those that are abelian ($C_2$, $C_3$), nilpotent ($D_8$, $H_3$), solvable ($S_3$, $S_4$) or non-solvable ($A_5$), as well as a physical rotation prediction problem based on $A_5$ (Sec.~\ref{sec:exp_rotation}).

Regarding the model architecture, in addition to the basic transformer \citep{trafo}, we study three types of diagonal structured SSMs: GLA \citep{YangWSPK24}, signed Mamba \citep{grazzi2024unlocking}, and AUSSM \citep{karuvally2025bridging}.
Generators of each models are positive, signed, and complex diagonal matrices, respectively,
thus can be viewed as discretized restricted SSMs.
DeltaProduct \citep{siems2025deltaproduct} is included as a general SSM class reference.

\begin{table}[t]
\setlength{\tabcolsep}{0.4em}
\caption{Length generalization performance measured as sequence-level accuracy for various categories of word problems. Models are trained on length 128 and tested on 256. $L$ denotes the number of layers. DeltaProduct uses the theoretically necessary number of Householder products depending on the problem. The blue color cell highlights perfect generalization.
The hyper-parameter search space is in App.~\ref{app:exp_wp}.
The number of training sequences is 500K except for $C_2$ which used 100K.
}
\vspace{-2mm}

\label{tab:word_problem}
\begin{center}
\small
\begin{tabular}{lccccccc}
\toprule
& & \multicolumn{2}{c}{Abelian} & \multicolumn{2}{c}{Nilpotent} & \multicolumn{2}{c}{Solvable}  \\
 \cmidrule(r){3-4} \cmidrule(r){5-6} \cmidrule(r){7-8} 
Model   & $L$      & \multicolumn{1}{c}{$C_2$} & \multicolumn{1}{c}{$C_3$} & \multicolumn{1}{c}{$D_8$} & \multicolumn{1}{c}{$H_3$} & \multicolumn{1}{c}{$S_3$} & \multicolumn{1}{c}{$S_4$} \\ \midrule
Transformer  & 1 & 0.00 & 0.00 & 0.00 & 0.00 & 0.00 & 0.00 \\ \midrule
(+-only) Mamba & 1 & 0.00 & 0.00 &  0.00  & 0.00 & 0.00 & 0.00  \\  \midrule
(+-only) GLA & 1 &  0.00  & 0.00 & 0.00 & 0.00 & 0.00 & 0.00 \\ 
                  & 2 &  0.00  & 0.00 & 0.00 & 0.00 & 0.00 & 0.00  \\ \midrule 
Signed Mamba          & 1 & \cellcolor{blue!10} 1.00 & 0.00 & 0.00 &  0.00  & 0.00 & 0.00   \\  
& 2 & \cellcolor{blue!10} 1.00 & 0.00 & \cellcolor{blue!10} 1.00 & \cellcolor{blue!10} 1.00 & \emph{0.19} & 0.00 \\  \midrule
AUSSM    & 1 & \cellcolor{blue!10} 1.00 & \cellcolor{blue!10} 1.00 & 0.00 &  0.00  & 0.00 & 0.00    \\  
    & 2 & \cellcolor{blue!10} 1.00 & \cellcolor{blue!10} 1.00 & 0.00 & 0.00 & 0.00 & 0.00 \\  \midrule
DeltaProduct & 1 & \cellcolor{blue!10} 1.00 &  \cellcolor{blue!10}  1.00 & \cellcolor{blue!10} 1.00  & \cellcolor{blue!10} 1.00 & \cellcolor{blue!10} 1.00 &\cellcolor{blue!10} 1.00   \\
\bottomrule
\end{tabular}
\end{center}
\vspace{-6mm}
\end{table}

\subsection{Length generalization on solvable word problems}
\label{sec:exp_wp}
We first evaluate sequence models on solvable WPs to test models' algebraic capacity.
Following prior work \citep{merrill2024illusion,siems2025deltaproduct}, the model predicts the finite group element obtained by evaluating the word (or, the prefix product) at each position. Models are trained on sequences of length 128 and tested on length up to 256.
Results are shown in Table \ref{tab:word_problem} and summarized as follows.

\textbf{Abelian WPs.}
We evaluate models on abelian WPs as a smoke test. As any abelian finite group embeds into abelian Lie group, a 1-layer abelian SSM can express abelian WPs in theory.

Empirically, echoing prior work, we observe that transformer, positive diagonal only (denoted as +-only) GLA, and Mamba failed to model the simplest abelian WP, $C_2$.
In contrast, Signed Mamba can solve $C_2$, while failing on $C_3$, aligned with \citet{grazzi2024unlocking}. AUSSM succeeded at both $C_2$ and $C_3$, matching its richer complex-valued diagonal structure.

\textbf{Nilpotent and Solvable WPs.}
We further evaluate models on solvable but non-abelian WPs.
Theoretically, $D_8$, $H_3$, and $S_3$ are extensions of two abelian groups. Following Thm.~\ref{thm:depth_extension}, these WPs can be represented by 2-layers of \emph{abelian} SSMs, while 1 layer is insufficient (Thm.~\ref{thm:sim_error}).
Consistent with the prediction, none of the 1-layer models except DeltaProduct successfully solved these tasks.

With 2 layers, we observe that the inductive bias significantly influences model behavior.
Signed Mamba successfully learned $D_8$ and $H_3$
and partially generalizes on $S_3$.
In contrast, AUSSM fails to learn a generalizable solution for
non-abelian word problems in our setting. This points to the \emph{learnability} issue:
our theory describes expressivity, but gradient-based optimization and architecture-specific inductive biases determine whether the solution is easy to find.
Signed Mamba has a more restricted eigenvalue range than AUSSM,
but empirically performed better in the $2$-layer regime.\footnote{see further comments on the trainability of AUSSM in App.~\ref{app:exp_rotation}}

$S_4$ is an extension of 3 abelian groups. As diagonal structured SSMs belong to \emph{restricted} SSM class, one may solve $S_4$ with 2 layers (Prop.~\ref{lem:stackderivedlength}). We conjectured that this would not be the case (Rem.~\ref{rem:conjecture_restricted_limits}).
\subsection{Depth vs.~length in non-solvable word problems.}
\begin{wrapfigure}{r}{0.55\textwidth}
    \begin{center}
        \includegraphics[width=0.53\columnwidth]{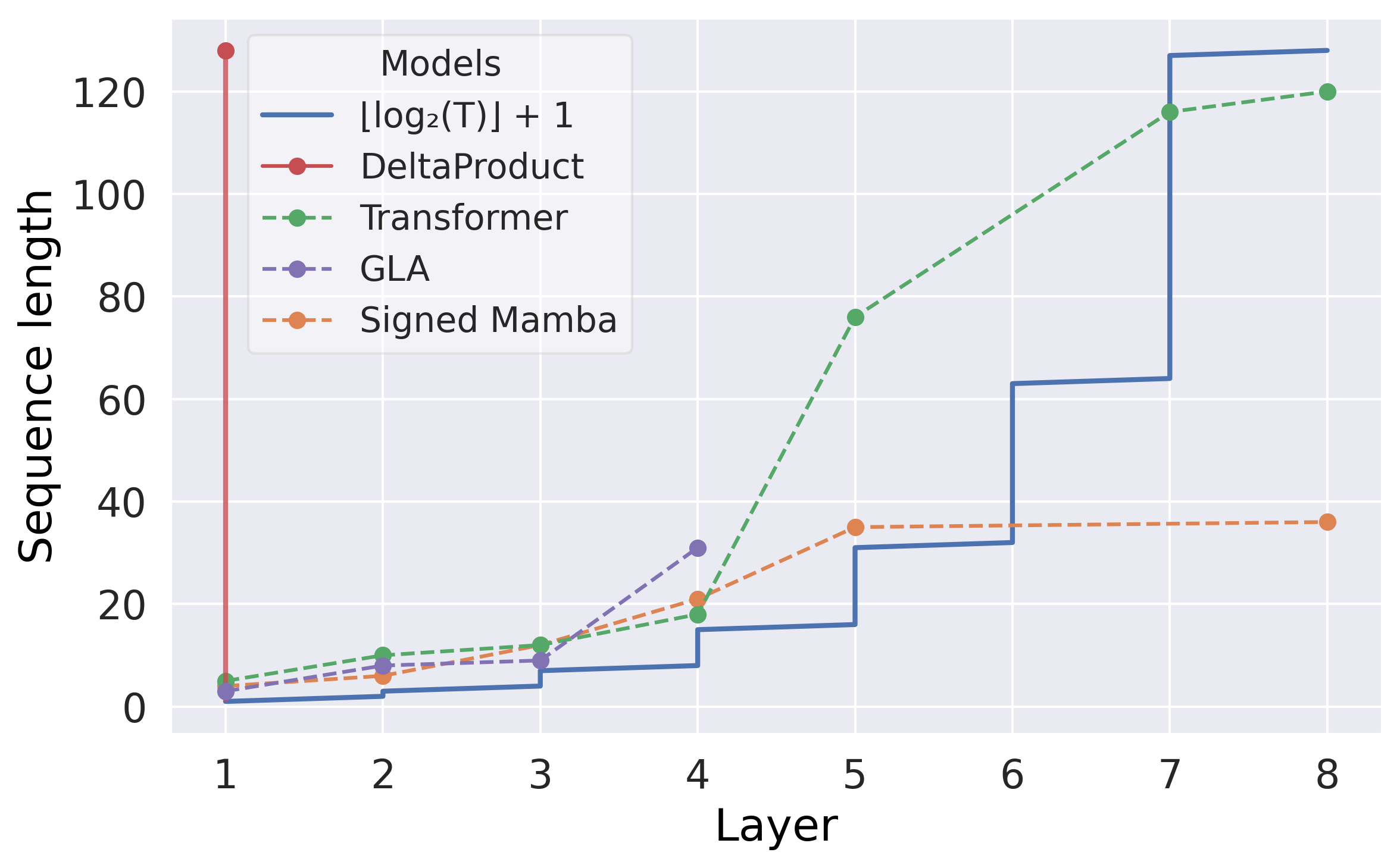}
        \caption{Maximum sequence length
        on the $A_5$ WP
        for each model with varying number of layers to achieve $>90\%$ training sequence accuracy.
        All models are trained on length up to 128. Deep models ($>$ 4 layers) that failed to achieve a longer sequence length than shallower models are not shown.
        $T$ in the legend $\bigl(\lfloor \log_2T \rfloor + 1\bigr)$ is sequence length (i.e., y-axis).
        The reference curve is the upper bound from Prop.~\ref{prop:logdepth}.
}
        \label{fig:seqlen_vs_layer}
    \end{center}
\end{wrapfigure}
Here we focus on $A_5$ as the smallest non-solvable WP to illustrate how the depth enhances performance
beyond models' expressivity range, in line with Prop.~\ref{prop:logdepth}.

Fig.~\ref{fig:seqlen_vs_layer} reports the maximum sequence length
for which each model achieves $>90\%$ accuracy at a given depth. We compared results to the theoretical depth upper bound from Prop.~\ref{prop:logdepth}.

We observe that transformer clearly benefits from depth.
By varying the depth from 1 to 8, transformer performance improves
following the trend of the theoretical bound.
GLA and Signed Mamba similarly improve up to depth 4 or 5.

However, we observe that deeper models are hard to train on this task; some deep models underperform shallower ones. 
For example, the 6- and 7-layer Signed Mamba did not improve over the 5-layer one, and 8-layer model performed only slightly better (achieving length 36 vs.~35).

What hinders the performance of deeper models remains an open question.
Consideration of several factors may offer some insight.
Our theory constructs on real arithmetic whereas empirical computation is on finite precision; the impact of finite precision is beyond the scope of the current work. 
Moreover, how end-to-end gradient-based optimization shapes the solution is unclear,
again illustrative of the practical learnability issue discussed in the previous section.
\subsection{Rigid-body rotation problem}
\label{sec:exp_rotation}
Finally, to demonstrate the relevance of our results beyond symbolic word problems, we conduct experiments with a continuous-valued state-tracking regression problem with explicit algebraic class.

The $A_5$ group acts as a rotational symmetry group of a dodecahedron.
Two chosen \emph{generators} $\{\mbf{P}, \mbf{R}\}$, each corresponding to the $120^\circ$ and $72^\circ$ rotation, generate $A_5$ group elements.
We chose a vector on a unit sphere as the initial state.
The model observes a sequence of $A_5$ elements, and the task is to predict the corresponding transformed vector after each rotation.
Algebraically, the problem corresponds to the action of $A_5$ embedded as a finite subgroup of $SO(3,\bb{R})$. Further experimental details are provided in App.~\ref{app:exp_rotation}.
\begin{figure}[t]
    \begin{center}
        \includegraphics[width=1.0\columnwidth]{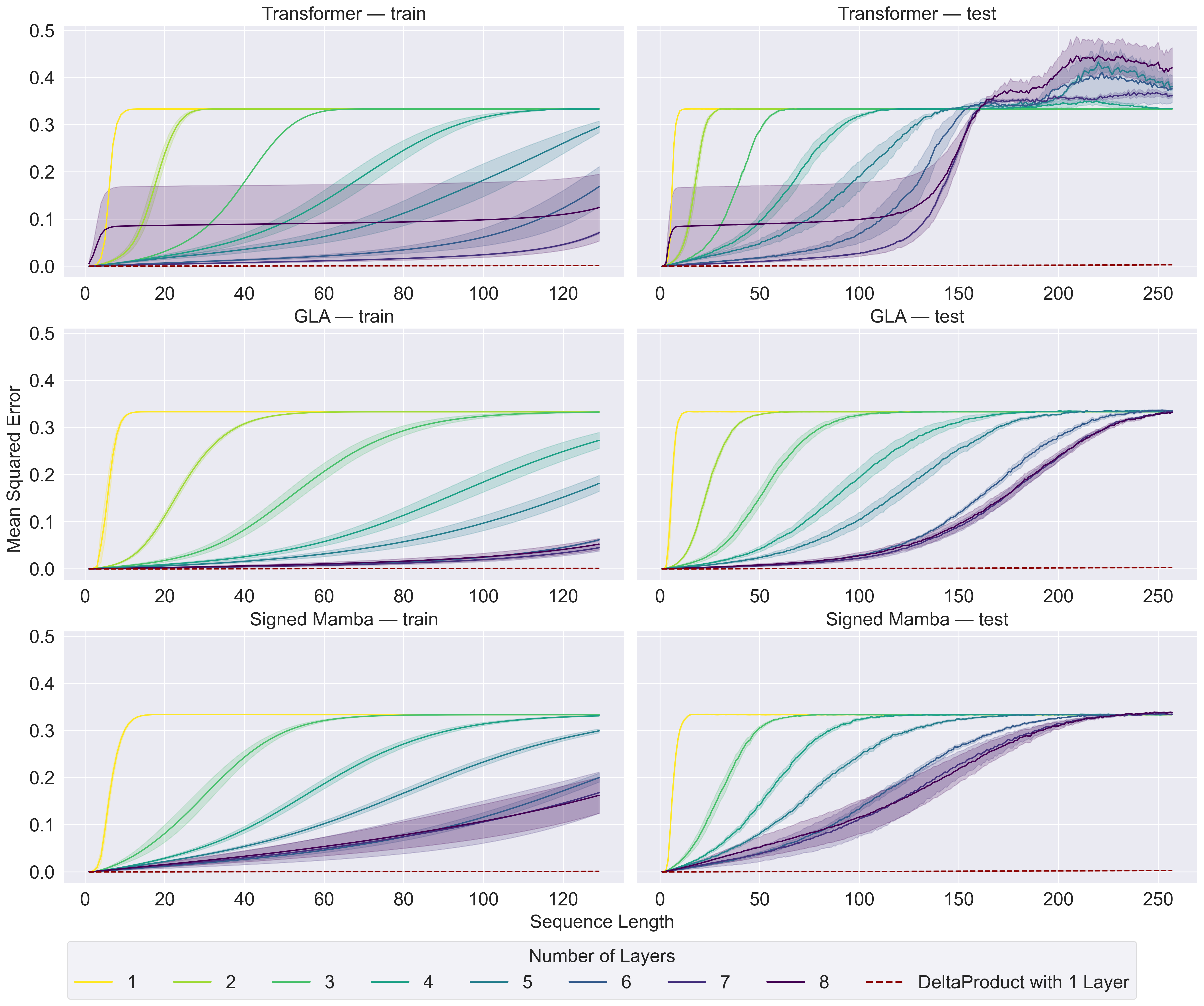}
        \caption{Mean squared error on the rotated vector prediction task at various sequence lengths for transformer, GLA, and signed Mamba, with different numbers of layers. Left and right columns illustrate the results on training and test sets, respectively. Standard error is shown using 3 seeds. Performance of DeltaProduct with 4 Householder products is shown as a reference.
}
        \label{fig:rotation_mse}
    \end{center}
     \vspace{-1mm}
\end{figure}

Fig.~\ref{fig:rotation_mse} shows prediction error as a function of depth for transformer, GLA, and Signed Mamba.\footnote{We found AUSSM to be generally unstable for this task; App.~\ref{app:exp_rotation} provides further discussions and corresponding results.}
Analogous to the symbolic tasks, depth plays a critical role in reducing the error.
However, learnability issues again appear.
The gains saturate for deep GLA and Signed Mamba beyond 6 layers, and deep transformer shows training variability. Taken together, the task supports the same perspective: depth improves order-sensitive approximation, while learnability remains a practical bottleneck.
\section{Discussion}
\label{sec:discussion}
Our theory connects the depth of sequence models to an abstract, algebraic extension theory.
The connection provides an optimistic message: when the task admits such an algebraic decomposition, depth gives an exact representation by composing simpler layers.
When the algebraic obstruction precludes an \emph{exact} simulation,
deep architectures compensate to rapidly minimize the error.
This conclusion is supported by our experiments. Across symbolic and physical tasks, we observe that increasing depth systematically reduces error on tasks outside each model's theoretical expressivity range. Our observations suggest depth as a structural mechanism that mitigates complex order sensitivity.
This also motivates parallelizable models with adaptive depth \citep{dehghani2018universal,csordas2021devil,csordas2021ndr,csordas2024moeut,tan2023sparse}.
\subsection{Future directions}
Our results naturally suggest several future directions.
One is to examine the impact of positional encoding (PE).
Although additive PE is unlikely to unlock expressivity \cite{liu2022transformers},
an advanced form of multiplicative PE does boost expressivity \cite{yang2025path}.
Understanding, or constructing, such mechanisms through algebraic structure
is a promising and practical avenue of research.

Orthogonally, our theory assumes real arithmetic and may be sensitive to the use of finite precision in real-world applications.  On the one hand, exponentially vanishing error may fall below numerical resolution, potentially blurring the algebraic obstruction. On the other hand, finite precision may constrain learnability of the system under gradient-based optimization, which we have not studied. How finite precision interacts with our findings and impacts expressivity needs to be deeply considered.

\subsection{Related Work}
\paragraph{Transformer variants}
The expressivity of sequence models is often studied via computational complexity and formal language theory \citep{GilesSCLC89,gers2001lstm,schmidhuber2001evaluating,weissGY18,hahn2020theoretical,MerrillWGSSY20,BhattamishraAG20,deletang2022neural,irie2023practical,StroblMW0A24,beck2024xlstm}.
\citet{liu2022transformers} showed that constant-depth transformers can simulate semiautomata with bounded sequence length via the Krohn-Rhodes theorem.
They further showed that no constant-depth transformer simulates general semiautomaton with arbitrary length sequence unless $\text{TC}^0 = \text{NC}^1$, and empirically showed the brittleness of the solution a transformer learns on fixed sequence length.
Subsequent work extends these limitations to Markov model variants, with depth scaling logarithmically in sequence length 
\citep{hu2024limitation}, or via graph theory to learn directed acyclic graph structures \citep{cheng2025transformers}.

\paragraph{Structured SSMs}
\citet{merrill2024illusion} extended circuit complexity theory to SSMs. They proved that constant-depth commuting SSMs, such as Mamba, admit $\text{TC}^0$ simulations and thus share similar expressivity limits with Transformers.
They also proposed architectural augmentations such as input-dependent, non-diagonal SSMs that can recover expressivity at increased computational cost.

Recent work aims to improve SSM expressivity at tolerable computational expense.
\citet{siems2025deltaproduct} leverage input-dependent compositions of generalized Householder transformations, grounded on \citet{schlag2021linear} and \citet{yang2024parallelizing},
with theoretical guarantees on regular language recognition.
Complementarily, another line of work leverages controlled differential equation (CDE) and rough path theory \cite{muca2024theoretical}.
Building on \citet{walker2024log},
\citet{walker2025structured} proposed structured linear CDEs (SLiCEs)
and proved their expressivity range.
Recent work \cite{anonymous2025the} proved $L$-layer complex-valued diagonal SSMs can track solvable groups of $\le L$ derived length, with empirical learnability gap.
\section{Conclusion}
Scalable sequence models achieve parallelism by imposing order symmetry, which limits their ability to \emph{exactly} solve many sequence problems. We make this performance gap measurable by Lie-algebraic control theory, and connect the depth of sequence models to algebraic extension theory. Our theory shows that, even when exact simulation is not possible, the order-sensitive approximation error vanishes exponentially with depth.
Experiments on symbolic word problems and 3D rotation state-tracking task supports our theory, while also exposing practical challenges such as the instability during training of deep models. Together, we provide an algebraic perspective for understanding depth-dependent error-expressivity in parallelizable sequence models.

\section{Limitations}
\label{sec:limitations}
Our work focuses on expressivity and error approximation, rather than learnability, optimization, or inductive biases which are other important challenges for sequence models.
While our theory offers guidance for model selection when the target task has an explicit algebraic structure, the exact algebraic class of real-world problems is rarely known.
In addition, our main contributions are of a  theoretical nature: we provide a theoretical framework to study how depth benefits certain sequence models. Although our theoretical results are relevant at any scale, large-scale experiments in language modeling or multimodal settings are out of scope.

\section*{Software and Data}
Our code is available at:\\
\url{https://github.com/kazuki-irie/lie-algebra-state-tracking}.

\begin{ack}
We thank Byungwoo Kang, T. Anderson Keller, Alireza A. Dehaqani, and David Sabatini for helpful feedback on the draft.
This work was partially supported by the Kempner Institute for the Study of Natural and Artificial Intelligence, a Polymath Award from the Schmidt Sciences, the NIH grant number R35NS137336, and the Department of Defense MURI program under ARO grant W911NF-231-0277.
\end{ack}

\bibliography{main, main_2}
\bibliographystyle{unsrtnat}

\appendix

\input{A1_Preliminaries}
\input{A2_Proofs}
\input{A3_Experiments}
\end{document}

%% file: A1_Preliminaries.tex
\section{Lie algebra for proofs}
\label{app:LieAlgebra}
In this section, we introduce further mathematical preliminaries
as necessary to provide self-contained proofs of our theorems.
We omit rigorous proofs for these observations,
and interested readers are directed to related textbooks and research papers \cite{iserles2000lie,lee2003smooth, petrogradsky2007wreath, tu2008manifold, rotman2009introduction, hall2013lie, robbin2022introduction,liu2022transformers}.
We open the section with the Lie group and Lie group action to motivate our Lie-algebraic approach.
\subsection{Lie groups and Lie algebras}
\label{app:subsubsec_Lieintro}
A set $G$ is a semigroup when $G$ is closed under an associative binary operation $\cdot: G \times G \to G$, usually called a "multiplication" operation.
If $\exists e \in G$ such that $e \cdot g = g \cdot e = g$ for $\forall g \in G$, $G$ is a monoid.
Moreover, $G$ is a group if for $\forall g \in G$ there exists an element $g^{-1} \in G$ such that $g^{-1} \cdot g = g \cdot g^{-1} = e$.
\paragraph{Lie group and Lie group action}
A Lie group $G$ is a smooth manifold equipped with a group structure under smooth multiplication.
For $x,y \in G$ let us denote the left translation $L_x(y) = xy$.
From the smoothness of the multiplication,
$L_x: G \to G$ is a diffeomorphism of the underlying manifold $\mcal{G}=G$.
In other words, translation can be thought of $G$ \emph{acting} on $\mcal{G}$.
The notion of Lie group \emph{action} generally follows:
\begin{definition}[Lie group action]
\label{app:def_LieGroupAction}
    Given a Lie group $G$ and a smooth manifold $\mcal{M}$,
    a (left) Lie group action is a smooth map
    $\Gamma: G \times \mcal{M} \to \mcal{M}$ that satisfies
    $\Gamma(e,h) = h$ and $\Gamma(xy,h) = \Gamma(x, \Gamma(y,h))$
    for an identity $e \in G$, any two elements $x, y \in G$,
    and a point on manifold $h \in \mcal{M}$.
\end{definition}
Following, we define the \emph{orbit} of $h$ denoted as $\text{Orb}_G(h) := \{\Gamma(g,h) \mid g \in G\}$.
\emph{Stabilizer}, or an isotropy subgroup, is a subgroup of $G$ where $\text{Stab}_G(h):= \{g\in G \mid \Gamma(g,h) = h \}$.

While Def.~\ref{app:def_LieGroupAction} straightforwardly describes how $\Gamma$ operates on $\mcal{M}$,
it is often advantageous to consider $\Gamma$ as inducing a diffeomorphism on $\mcal{M}$.
We can naturally introduce a Lie group homomorphism
$\Pi: G \to \text{Diff}(\mcal{M})$ such that $\Pi(g)=\Gamma(g,\cdot)$,
where $\text{Diff}(\mcal{M})$ denotes collection of diffeomorphisms of $\mcal{M}$.
\paragraph{Lie bracket and Lie algebra}
Now one may consider a smooth vector field $V: G \to TG$
and their collection $\fr{X}(G)$.
Given the differential of the left translation at point $y$,
\begin{equation*}
    (dL_x)_y: T_yG \to T_{L_x(y)}G = T_{xy}G,
\end{equation*}
the pushforward $(L_x)_*V \in \fr{X}(G)$ is well-defined \cite{tu2008manifold,lee2003smooth}.
Then, $V$ is said to be left-invariant when $(L_x)_*V = V$.
In other words, a left-invariant vector field $V$
satisfies $\bigl((L_x)_*V\bigr)_{xy} = (dL_x)_yV_y = V_{xy}$ for $\forall x, y \in G$
and completely determined by its evaluation at identity, as $V_x = V_{xe} = (dL_x)_e V_e$.
Thus, a collection of left-invariant vector fields, $\fr{X}_L(G)$, has a one-to-one correspondence with $T_e G$.

As a tangent vector is defined by how it differentiates smooth functions,
a smooth vector field on $\mcal{M}$
is commonly identified as the derivation $C^{\infty}(\mcal{M}) \to C^{\infty}(\mcal{M})$.
Following Leibniz rule,
the Lie bracket naturally arises to satisfy the derivation property \cite{tu2008manifold}.
For $f, g \in C^{\infty}(\mcal{M})$,
\begin{align*}
    XY(fg) &= X((Yf)g + fYg) \\
    &= (XYf)g + (Yf)(Xg) + (Xf)(Yg) + f(XYg) \\
    &\neq (XYf)g + f(XYg).
\end{align*}
Observe the extra term $(Yf)(Xg) + (Xf)(Yg)$ obstructs the derivation property.
One can subtract $YX(fg)$ to cancel out the extra term,
and naturally $(XY-YX)$ is again a derivation of $C^{\infty}(\mcal{M})$.
\begin{definition}[Lie bracket of vector fields]
    Given two smooth vector fields $X, Y: M \to TM$ on an open subset $M$ of a smooth manifold $\mcal{M}$,
    their Lie bracket is defined as following:
    \begin{equation}
    \label{app:eq_LieBracketDef}
        [X,Y] f = (XY - YX)f = X(Yf) - Y(Xf), \quad \forall f \in C^\infty(M).
    \end{equation}
\end{definition}
$\fr{X}_L(G)$ is closed under the Lie bracket.
Now we are ready to associate the Lie algebra $\fr{g}$ to the Lie group $G$:
\begin{definition}[Lie algebra associated to Lie group]
\label{app:LieAlgebra_Tangent}
    For a Lie group $G$, we identify the vector space of Lie algebra $\fr{g}:=T_eG$.
    For any $X \in \fr{g}$, assign a unique $X^L \in \fr{X}_L(G)$ that satisfies $X^L(y) = (dL_y)_eX$ for $y \in G$.
    We define the Lie bracket operation on $\fr{g}$ such that for $X,Y \in \fr{g}$,
    \begin{equation*}
        [X,Y] := [X^L,Y^L](e).
    \end{equation*}
    Then a map $\pi^L:\fr{g} \to \fr{X}_L(G)$, $\pi^L(X) = X^L$ is a Lie algebra isomorphism.
\end{definition}
Naturally, Lie algebra $\fr{g}$ can also be identified as a collection of the left-invariant, first-order differential operators on $G$.
Note the Lie bracket operation depends on the identification: with $\fr{g} \cong T_eG$, $[X,Y] \in T_eG$.
On the other hand, identifying $\fr{g} \cong \fr{X}_L(G)$ produces $[X,Y] \in \fr{X}_LG$.
The identification should be clear from the context.

We can analogously define the right translation on Lie group $R_x(y) = yx$.
The differential of the right translation at point y gives $(dR_x)_y: T_yG \to T_{R_x(y)}G = T_{yx}G$,
and we denote the collection of right-invariant vector fields as $\fr{X}_R(G)$.
Notationally, one can prove that $[X^L, Y^L](e) = -[X^R, Y^R](e)$ \cite{lee2003smooth}.
In that, $\pi^R: \fr{g} \to \fr{X}_RG$, $\pi^R(X) = X^R$ swaps the Lie bracket sign.

\paragraph{Bridge to geometry}
As a concrete example, consider the open set $M$ in Eq.~\ref{app:eq_LieBracketDef} with local coordinate charts $x=(x^1, \cdots, x^n)$
around point $p \in M$.
Then we can rewrite $X \in \fr{X}(M)$ in a coordinate expression:
\begin{equation*}
    X = \sum_{a=1}^n X^a \partial_a,
    \quad
    X^a = X(x^a) \in C^\infty(M),
\end{equation*}
and similarly for $Y \in \fr{X}(M)$.
Given the local coordinate, Lie brackets get a tangible interpretation.
One can exhaust the vector field Lie brackets with coordinate,
\begin{align*}
    [X,Y] = \sum_{c=1}^n [X,Y]^c \partial_c,
    \quad
    [X,Y]^c = X(Y^c) - Y(X^c),
\end{align*}
where $[X,Y]^c$ is again a coordinate function of $[X,Y] \in \fr{X}(M)$.
Evaluated at $p$, $[X,Y]^c(p) = X_p(Y^c) - Y_p(X^c)$.
Then from the differential it is transparent:
\begin{equation}
\label{app:eq_LieBracket_Directional}
    [X,Y]^c(p) = X_p(Y^c) - Y_p(X^c) = (dY^c)_p(X_p) - (dX^c)_p(Y_p).
\end{equation}
One may interpret above as a directional derivative of a smooth coordinate function
along the tangent vector direction \cite{lee2003smooth}.

\subsection{Representation and infinitesimal action}
\label{app:subsubsec_repprelim}
\paragraph{Lie group representation}
The abstract notion provided at App.~\ref{app:subsubsec_Lieintro} is often easier to work with a tangible object,
and matrix is one of such.
As $g, g^{-1} \in G$, let us consider the collection of any invertible matrices in $\bb{R}^{n \times n}$.
Over the real field, we name the collection as the general linear group and denotes as $\G{GL}{n;\bb{R}}$.
Matrix Lie group is defined as a closed subgroup of $\G{GL}{n;\bb{R}}$,
and many Lie groups in application arise as one.

These are the case of \emph{Lie group representations} \cite{lee2003smooth, hall2013lie}.
Given a real vector space $V$,
a representation of Lie group $G$ is a Lie group homomorphism $\Pi: G \to \G{GL}{V}$.
When $\dim V = n$, we can identify $\G{GL}{V}$ with $\G{GL}{n;\bb{R}}$ via basis choice.
Recall Def.~\ref{app:def_LieGroupAction} that representation is just a linear Lie group action of $G$ on $V$.
\paragraph{Lie algebra representation}
With Def.~\ref{app:LieAlgebra_Tangent} we can think of the corresponding \emph{Lie algebra representations}.
For a Lie group $G$ associated with $\fr{g}$,
consider $X \in \fr{g} \cong T_eG$ and corresponding $X^L \in \fr{X}_L(G)$.
Let $c_X: \bb{R} \to G$ be the integral curve of $X^L$ with $c_X(0) = e \in G$, starting at the identity.
Then we can construct the exponential map,
which is a canonical smooth map from the Lie algebra to its integrating group \cite{lee2003smooth}:
\begin{equation*}
    \exp: \fr{g} \to G, \quad \exp(X):= c_X(1),
\end{equation*}
and we write $\exp(tX) = c_X(t)$ so that $\frac{d}{dt}\vert_{t=0} \exp(tX) = X$.
Notice that $X$ corresponds to the initial velocity of $c_X(t)$.

One concrete example is the matrix Lie algebra $\fr{gl}(n,\bb{R})$, which is a collection of any matrices in $\bb{R}^{n \times n}$.
With the Lie group representation:
\begin{equation*}
    \Pi: G \to \G{GL}{n;\bb{R}}, \quad \Pi_*: T_eG \to T_{\Pi(e)}\G{GL}{n;\bb{R}} = T_{I_n}\G{GL}{n;\bb{R}},
\end{equation*}
where we identified $\G{GL}{n;\bb{R}}$-associated Lie algebra
$\fr{gl}(n,\bb{R}) \cong T_{I_n}\G{GL}{n;\bb{R}}$.
Then the Lie algebra representation is naturally induced:
\begin{equation*}
    \pi:= (d \Pi)_e: \fr{g} \to \fr{gl}(n,\bb{R}),
\end{equation*}
in that, $\pi$ is the differential of $\Pi$ at the identity.
For the matrix Lie algebra, exponential map is the matrix exponential $\text{expm}$:
\begin{equation}
\label{app:eq_matrixLie_exponential_action}
    \Pi(\exp(X)) = \text{expm} (\pi(X)), \: \pi(X) = \frac{d}{dt}\bigg\vert_{t=0}\Pi(\exp(tX)).
\end{equation}
We extend the differential of representation to the differential of action,
or \emph{infinitesimal} action, in the following paragraph.
\paragraph{Fundamental vector field}
Consider the Lie group $G$ acts on itself by left multiplication.
For $h \in G$ and $X \in \fr{g}$, let us define a vector field $X^\sharp$:
\begin{equation*}
    X^\sharp(h) = \frac{d}{dt} \bigg \vert_{t=0} \exp(-tX)h = -(dR_h)_eX \in T_hG.
\end{equation*}
Observe that $X^\sharp(h)$ recovers familiar ODE $\dot{h} = Xh$
up to sign change, when $G \subseteq \G{GL}{n,\bb{R}}$.

We can extend above to an arbitrary smooth finite manifold $\mcal{M}$,
replacing canonical left action on $G$ to a chosen left Lie group action $\Gamma: G \times \mcal{M} \to \mcal{M}$.
For each $X \in \fr{g}$ and $h \in \mcal{M}$, we define the fundamental vector field $X^\sharp \in \fr{X}(\mcal{M})$:
\begin{equation}
\label{app:eq_fundamental_vector_field}
    X^\sharp (h) := \frac{d}{dt}\bigg\vert_{t=0}\Gamma(\exp(-tX), h) \in T_h\mcal{M},
\end{equation}
then the map $\gamma: \fr{g} \to \fr{X}(\mcal{M})$,
$\gamma(X)= X^\sharp$ is a Lie algebra homomorphism \cite{iserles2000lie, robbin2022introduction}.
Therefore $\gamma$ is an \emph{infinitesimal} Lie group action, and considered as Lie algebra action on $\mcal{M}$.
\begin{remark}
\label{app:rem_Bracket_Sign}
Many text writes $X^\sharp$ in a \emph{right-invariant} vector field form \cite{robbin2022introduction}:
\begin{equation}
    X^\sharp (h) := \frac{d}{dt}\bigg\vert_{t=0}\Gamma(\exp(tX), h) \in T_h\mcal{M},
\end{equation}
and the map $\gamma(X)= X^\sharp$ is a Lie algebra \emph{anti}homomorphism \cite{iserles2000lie, robbin2022introduction}.
Under this definition, one would observe that $[\pi(X), \pi(Y)] = \pi([X,Y])$, a matrix commutator,
aligns with $[X^\sharp, Y^\sharp] = -[X,Y]^\sharp$ up to sign change.
In practice, the sign can be absorbed and need not be worried.
Some literature swap the Lie bracket direction to handle the sign change,
but this is just a choice of convention and does not obstruct any fundamental logic.
\end{remark}
\paragraph{Bridge to homogeneous ODE}
\label{app:bridge_to_hom_LPV}
As observed above, when a Lie group $G \subseteq \G{GL}{n,\bb{R}}$ left-acts on a smooth manifold $\mcal{M} \cong \bb{R}^n$,
the fundamental vector field recovers homogeneous ODE.
For a point $h \in \mcal{M}$ and $\mbf{A} \in \fr{g} \subseteq \fr{gl}(n,\bb{R})$, we select $\Gamma$ to be a left multiplication.
Then obviously:
\begin{equation*}
    \dot{h} = \frac{d}{ds}\bigg\vert_{s=0}\Gamma(\exp(s\mbf{A}), h) = \mbf{A}h,
\end{equation*}
which corresponds to the linear controlled homogeneous ODE, $\dot{h}(t) = \mbf{A}(x(t))h(t)$, given a fixed input $x(t)$.

Critically, when $G$ acts on itself, we can recover the state-transition matrix $\Phi$.
Recall Def.~\ref{def:Phi}, $\Phi$ satisfies:
\begin{equation}
\label{app:eq_homogeneous_matrix_state}
    \dot{\Phi}(t,0) = \mbf{A}(x(t))\Phi(t,0), \quad \Phi(0,0) = I_n,
\end{equation}
then Eq.~\ref{app:eq_homogeneous_matrix_state} forms a homogeneous ODE system with $\Phi$ being its latent state.
This observation is very useful:
it is immediate that for any admissible input, $\Phi(t,0) \in G$.
In a sense, one may understand Eq.~\ref{app:eq_homogeneous_matrix_state}
as tracking \emph{every} possible state evolution, rather than a single state.
This is a conventional perspective in geometric control theory,
commonly referred as control theory on Lie groups \cite{jurdjevic1972control, iserles2008magnus}.

\subsection{Lie algebra extension}
\label{app:subsubsec_extensionprelim}
Analogous to the notion of the group extension, Lie algebra extension formalizes \emph{gluing} two Lie algebras together.
Here we introduce the notion of Lie algebra extension \cite{weibel1994introduction,rotman2009introduction, PhysicsExtension, beckett2022symplectic}.
\begin{definition}[Lie algebra extension]
    \label{app:def_LieAlgebraExtension}
    The Lie algebra $\tilde{\fr{h}}$ is an extension of base Lie algebra $\fr{h}$ by (finite-dimensional) kernel Lie algebra $\fr{a}$ if there exists a short exact sequence of Lie algebras:
    \begin{equation}
    \label{app:eq_LieAlgebraExtension_shortsequence}
        0 \to \mathfrak{a} \xrightarrow{i} \tilde{\mathfrak{h}} \xrightarrow{\theta} \mathfrak{h} \to 0.
    \end{equation}
\end{definition}
Redundantly, $i: \fr{a} \to \tilde{\fr{h}}$ is a Lie algebra monomorphism,
while $\theta: \tilde{\fr{h}} \to \fr{h}$ is a Lie algebra epimorphism.
The short exact sequence satisfies $\text{im}\; i = \ker \theta$,
thus $i(\fr{a})$ is an ideal in $\tilde{\fr{h}}$ and $\tilde{\fr{h}}/i(\fr{a}) \cong \fr{h}$.
Following common convention, we identify $\fr{a}$ with $i(\fr{a})$ without confusion.
\begin{remark}[Tower of abelian extension]
\label{app:rem_solvable_k_tower}
A solvable Lie algebra with derived series $\fr{g} = \fr{g}^{(0)} \supset \fr{g}^{(1)} \supset \cdots \supset \fr{g}^{(k)}=\{0\}$
is an extension:
\begin{equation*}
    0 \to \fr{g}^{(n)} \to \fr{g} \to \fr{g} / \fr{g}^{(n)} \to 0.
\end{equation*}
When $n = k-1$, $\fr{g}^{(k-1)}$ is an abelian Lie (sub-)algebra and
$\fr{g} / \fr{g}^{(k-1)}$ is a solvable Lie algebra of derived length $k-1$.
In iteration, we recover a $k$-length \emph{tower} of abelian Lie algebra extension.
\paragraph{Split extension}
Understanding how $\fr{h}$ acts on $\fr{a}$ is critical in characterizing $\tilde{\fr{h}}$.
We begin with a simple case.
\end{remark}
\begin{definition}[Split extension]
    In Eq.~\ref{app:eq_LieAlgebraExtension_shortsequence}
    consider a linear section $s: \fr{h} \to \tilde{\fr{h}}$ that satisfies $\theta \circ s = 1_\fr{h}$, an identity map.
    If $s$ is a Lie algebra homomorphism, Eq.~\ref{app:eq_LieAlgebraExtension_shortsequence} is \emph{split}.
\end{definition}
If Lie algebra extension splits, then
we say $\tilde{\fr{h}}$ is a \emph{semidirect product} (or sum) of $\fr{h}$ by $\fr{a}$
and we denote as $\tilde{\fr{h}} = \fr{h} \ltimes \fr{a}$
(or $\fr{a} \rtimes \fr{h}$) \cite{rotman2009introduction}.\footnote{Note
that the convention of sign direction is not accordant across references \cite{petrogradsky2007wreath, rotman2009introduction}.
here we emphasize that $\fr{h}$ acts on $\fr{a}$.}

Consider $X,Y \in \fr{h}$ and $u,v \in \fr{a}$.
$\tilde{\fr{h}}$ acts on $\fr{a}$ by derivation, as $\fr{a}$ is an ideal in $\tilde{\fr{h}}$.
As the linear section $s$ is a Lie algebra homomorphism,
$s(X) \in \tilde{\fr{h}}$ naturally acts on $\fr{a}$ by derivation too.
Thus, the induced action of $\fr{h}$ on $\fr{a}$ is straightforward here:
\begin{equation*}
    \rho: \fr{h} \to \text{Der}(\fr{a}),
    \quad
    \rho(X)v = [s(X),v],
\end{equation*}
and Lie bracket on split extension follows:
\begin{align}
\label{app:eq_split_bracket}
    [(X, u), (Y, v)]_{\tilde{\fr{h}}} &= ([X,Y]_\fr{h}, [u,v]_\fr{a} + \rho(X)v - \rho(Y)u).
\end{align}
One special case of split extension is the trivial extension where $[(X,u),(Y,v)] = ([X,Y],[u,v])$,
and we name it as a \emph{direct product} (or sum) of $\fr{h}$ and $\fr{a}$.

\paragraph{Bridge to inhomogeneous ODE}
\label{app:bridge_to_general_LPV}
Similar to App.~\ref{app:bridge_to_hom_LPV}, affine vector field naturally arises from the semidirect product.
Consider an inhomogeneous ODE $\dot{h}(t) = \mbf{A}(x(t))h(t) + \mbf{b}(x(t))$.
As noted often \cite{krener1975bilinear, krener1977decomposition, iserles2000lie},
one can augment the state to \emph{homogenize} the ODE:
\begin{equation*}
    \bar{h}(t) = \begin{pmatrix}
    h(t) \\
    1
\end{pmatrix},
\quad
\dot{\bar{h}}(t)
= \begin{pmatrix}
        \mbf{A}(x(t)) & \mbf{b}(x(t)) \\
        0 & 0
    \end{pmatrix}
    \:
    \bar{h}(t).
\end{equation*}
Then it is immediate that the generator class is affine.
Consider a finite-dimensional Lie group of affine transformation, $\G{Aff}{n, \bb{R}}$.
Its associated Lie algebra $\fr{aff}(n, \bb{R})$ is the canonical example of the split extension:
\begin{equation*}
    0 \to \bb{R}^n \to \fr{aff}(n, \bb{R}) \to \fr{gl}(n,\bb{R}) \to 0,
\end{equation*}
such that $(A,b) \in \fr{gl}(n,\bb{R}) \ltimes \bb{R}^n$.
With a standard matrix embedding,
\begin{equation*}
    (A,b) \longmapsto \begin{pmatrix}
        A & b \\
        0 & 0
    \end{pmatrix},
\end{equation*}
which corresponds to the generator class of the homogenized ODE above.

Let us confirm that realized SSM respects the split extension bracket rule in Eq.~\ref{app:eq_split_bracket}.
From the matrix representation above, one can simply check their matrix commutator.
With two fixed inputs $x$, $y$:
\begin{equation*}
    \bar{\mbf{A}}(x) = \begin{pmatrix}
        \mbf{A}(x) & \mbf{b}(x) \\
        0 & 0
    \end{pmatrix},
    \quad
    [\bar{\mbf{A}}(x), \bar{\mbf{A}}(y)]
    =
    \begin{pmatrix}
        [\mbf{A}(x), \mbf{A}(y)]
        &
        \mbf{A}(x)\mbf{b}(y) - \mbf{A}(y)\mbf{b}(x) \\
        0 & 0
    \end{pmatrix},
\end{equation*}
especially when $\text{Lie}{\{\mbf{A}(x)\}}$ is abelian, $[\mbf{A}(x), \mbf{A}(y)] = 0$ and only the top right entry is nonzero.

Now consider vector field representation. For the brevity, let us keep $\text{Lie}{\{\mbf{A}(x)\}}$ be abelian.
Recall Eq.~\ref{app:eq_LieBracket_Directional}:
\begin{align}
X(h) &= \mbf{A}(x)h + \mbf{b}(x), \quad Y(h) = \mbf{A}(y)h + \mbf{b}(y) \nonumber \\
    [X,Y]^c &= \sum_{a=1}^{n}
    \left(
    \mbf{b}^a(x)\mbf{A}^{ca}(y)
    - \mbf{b}^a(y)\mbf{A}^{ca}(x)
    \right) \nonumber \\ 
    [X,Y] &= \mbf{A}(y)\mbf{b}(x) - \mbf{A}(x)\mbf{b}(y),\label{app:eq_LPV_affinebracket}
\end{align}
with sign change noted in Rem.~\ref{app:rem_Bracket_Sign}.

Taken together with App.~\ref{app:bridge_to_hom_LPV},
switching the vector state space
with the corresponding Lie group, or flow space,
is often referred as \emph{lifting} \cite{krener1977decomposition}.
Except for minor details in implementation, lifting does not lose any generality
and we can \emph{project} back to the state space \cite{iserles2000lie, iserles2008magnus}.
\paragraph{Central extension and beyond}
If all extension belongs to the split extension class, our discussion would be much easier.
However note that the linear section $s: \fr{h} \to \tilde{\fr{h}}$ needs not be a Lie algebra homomorphism.
Among non-split extension, central extension is another key class of the Lie algebra extension.
\begin{definition}[Central extension]
    In Eq.~\ref{app:eq_LieAlgebraExtension_shortsequence} let $\fr{a}$ be abelian.
    When $i(\fr{a})$ is in the center of $\tilde{\fr{h}}$ such that $[\tilde{X}, u] = 0$ for any $\tilde{X} \in \tilde{\fr{h}}$
    and $u \in i(\fr{a})$, we call it as central extension.
\end{definition}
As $s$ is not necessarily Lie algebra homomorphism, we introduce a correction term, called \emph{2-cocycle} \cite{weibel1994introduction}:
\begin{equation}
\label{app:eq_2cocycle}
    \omega: \fr{h} \times \fr{h} \to \fr{a}, \quad
    \omega(X,Y) := [s(X),s(Y)] - s([X,Y]_\fr{h}).
\end{equation}
Recall Eq.~\ref{app:eq_split_bracket}.
As $\fr{a}$ belongs to the center, the induced action $\rho: \fr{h} \to \text{Der}(\fr{a})$ is trivial.
Thus, in central extension, the bracket rule simplifies to:
\begin{equation*}
\label{app:eq_central_bracket}
[(X,u),(Y,v)] = ([X,Y], \omega(X,Y)).
\end{equation*}
When $s$ is Lie algebra homomorphism, central extension breaks down to trivial extension.

Taken together, the Lie bracket on general Lie algebra extension follows:
\begin{equation*}
    [(X, u), (Y, v)]_{\tilde{\fr{h}}} = [[X,Y]_\fr{h}, [u,v]_\fr{a} + \rho(X)v - \rho(Y)u + \omega(X,Y)].
\end{equation*}
Characterization of general Lie algebra extension is nontrivial.
At least under Lie group setting, the Kaloujnine-Krasner universal embedding theorem states that
any Lie group extensions embeds into the \emph{wreath} product of Lie group \cite{wells1976some,deval2024universal}.
Fortunately, following studies have extended the universal embedding theorem
to Lie algebra \cite{petrogradsky2007wreath, deval2024universal}.
\section{Control theory basics}
\label{app:LPV}
\subsection{Abstract control systems}
\label{app:control_detail}
Here we discuss abstract notions of controlled dynamical systems \cite{iserles2000lie}.
On a real smooth finite-dimensional manifold $\mcal{M}$, we consider a time-dependent tangent vector field:
\begin{equation*}
    F: [0,\tau] \times \mcal{M} \times \mcal{X} \to T\mcal{M},
\end{equation*}
which is locally Lipschitz in $h \in \mcal{M}$ and measurable in time $t \in [0,\tau]$.  
In this context, we call $\mcal{M}$ the \emph{state space} and $\mcal{X}$ the \emph{input space}.
Then, with a fixed input path $x: [0,T] \to \mcal{X}$ we obtain the controlled ODE:
\begin{equation}
    \label{eq:abstractODE}
    \dot{h}(t) = F_x(t,h(t)), \; F_x(t,h(t)) := F(t,h(t),x(t)),
\end{equation}
assuming the path is locally bounded and integrable \cite{iserles2000lie}.
We call a tuple $S = (\mcal{M}, F,\mcal{X})$ a \emph{controlled dynamical system}.
Lastly, we say the system is initialized when $h(0)$ is fixed.
Restricting $\mcal{M}$ to be on a Euclidean state space realizes SSM.
\subsection{Simulation}
\label{app:sim_detail}
Here we formally introduce the notion of simulation \cite{sussmann1976existence,hermann1977nonlinear, krener1977decomposition}.
\begin{definition}[Simulation]
\label{def:SSM_sim_equivalence}
    Let two initialized $S_0$, $S_1$ share an input space.
    Consider a smooth map $\varphi: \mcal{M}_0 \to \mcal{M}_1$
    such that $\varphi(h_0(t)) = h_1(t)$ for any admissible inputs.
    If $\varphi$ is a surjective submersion, we say $S_0$ \emph{simulates} $S_1$.
    Moreover, if $\varphi$ is (locally) diffeomorphic, we say $S_0$ is equivalent to $S_1$.
\end{definition}
Note that if $S_0$ simulates $S_1$ and $\dim S_0 = \dim S_0$,
then $\varphi$ is diffeomorphic.
Therefore, equivalence implies simulation.

For SSMs, $\varphi$ is restricted to be full-rank linear map.
Following, diffeomorphic $\varphi$ is invertible linear map.
Arbitrary smooth $\varphi$ would typically leave the SSM model class \cite{krener1977decomposition,toth2010modeling}.
When restricted to SSMs, we define approximation error correspondingly.
\begin{definition}[SSM approximation error on horizon]
\label{def:SSM_sim_error}
    Consider initialized $\mbf{S}_0$, $\mbf{S}_1$ with $n_0 = \dim \mbf{S}_0 \ge \dim \mbf{S}_1 = n_1$.
    Let $\mcal{X}_{ad}$ be shared admissible input path space, with
    \begin{equation*}
        U := \{x \in \mcal{X}_{ad} \mid x: [0,T] \to \mcal{X}, \|x\| \le M\}
    \end{equation*}
    be measure bounded paths over chosen interval $[0,T]$.
    Then, we define approximation error:
    \begin{equation}
    \label{eq:SSM_sim_error}
    \Delta_h := \inf_{\mbf{P} \in \bb{R}^{n_1 \times n_0}} \sup_{x \in U} \bigl\| \mbf{P}h_0(T) - h_1(T) \bigr\|,
    \end{equation}
    where $\mbf{P}$ is a full-rank linear map.
\end{definition}
Note that $\mbf{P}$ can be replaced with invertible linear map $\mbf{T}$
by state augmentation on $\mbf{S}_1$, which leads to the standard state-space equivalence relation
in linear parameter-varying (LPV) system theory \cite{toth2010modeling}.
\subsection{Minimality}
\label{app:min_detail}
To understand whether one system simulates the other, we need a notion of minimality.
With a concept of Lie group action, minimality comes in a straightforward manner \cite{jurdjevic1972control, krener1977decomposition}.

Consider $5$-tuple $(\mcal{M},F,\mcal{X}_{ad},\mcal{Y},g)$
and its evolution operator $\Psi$.
Consider a collection $S = \{\Psi_x\}$ for $x \in \mcal{X}_{ad}$,
then one can notice $S \subset \text{Diff}(\mcal{M})$ is a semigroup.
Although any $\Psi_x(t,0)$ represent as an invertible matrix,
the existence of admissible input for $\Psi_{x'} = \Psi_x^{-1}$ is not guaranteed.
With completion, we denote a group $G$ be the smallest subgroup of $\text{Diff}(\mcal{M})$ containing $S$.

For a chosen point $h \in \mcal{M}$, consider a collection $S(h) = \{\Psi(h): \Psi \in S\}$ and $G(h) = \{\Psi(h): \Psi \in G\}$.
It is akin to the orbit of Lie group action:
intuitively, $S(h)$ is the set of accessible points on $\mcal{M}$ under admissible inputs.
We say the system is (weakly) \emph{controllable} at point $h$ when $S(h) = \mcal{M}$.
When satisfied, the condition implies $G(h) = \mcal{M}$.

Now consider the output map.
Given any two points $h_1, h_2 \in \mcal{M}$, denote $y_1(t) = g(\Psi(t,0)(h_1))$ and $y_2(t)$ correspondingly.
Then $h_1$ and $h_2$ are said to be indistinguishable when $y_1(t) = y_2(t)$ for any admissible inputs.
The system is \emph{observable} when $h_1$ and $h_2$ being indistinguishable implies $h_1 = h_2$.
Collectively, we say the system $(\mcal{M},F,\mcal{X}_{ad},\mcal{Y},g)$ is \emph{minimal}
when both (weakly) controllable and observable \cite{krener1977decomposition}.
\begin{remark}
    Often, the attention is restricted to open subset $U \subseteq \mcal{M}$ however small.
    Without global Lie group action and orbits, one can characterize the system with Lie algebras.
    For example, weak controllability is replaced by \emph{Lie algebra accessibility rank condition} \cite{sontag2013mathematical}.
\end{remark}
One may easily notice the notion of minimality in LTI system is a special case of the above.
We can similarly induce notion of minimality for SSMs,
commonly studied as a (quasi-)LPV system \cite{toth2010modeling}.
\subsection{Magnus series and path signature}
\label{app:CDE}
The Magnus series provides Lie-algebraic expansion of the state-transition matrix.
As various evolution operators in physical applications are indeed a integral curve on a Lie group,
the Magnus expansion (or its truncation) serves as a useful tool in numerical approximation problems \cite{iserles2000lie}.
Given a system $\mbf{S}$, we explicitly write down a few first terms:
\begin{align*}
    \Phi(t,0) &= \exp \Omega(t),
    \quad
    \Omega(t) = \sum \Omega_i(t)\\
    \Omega_1(t) &= \int_0^t \mbf{A}(t_1) \; dt_1 \\
    \Omega_2(t) &= \frac{1}{2}\int_0^t \; dt_1 \int_0^{t_1} \; dt_2
    [\mbf{A}(t_1), \mbf{A}(t_2)] \\
    \Omega_3(t) &= \frac{1}{6}\int_0^t \; dt_1 \int_0^{t_1} \; dt_2 \int_0^{t_2} \; dt_3
    \Bigl(
    [\mbf{A}(t_1), [\mbf{A}(t_2), \mbf{A}(t_3)]] + [\mbf{A}(t_3), [\mbf{A}(t_2), \mbf{A}(t_1)]]
    \Bigr)
    \\
    \cdots
\end{align*}
Observe how the lower central series contributes to each order of the Magnus series.
Specifically, the $i$-th order Magnus term involves $(i-1)$ Lie bracket operations. In the lower central series, these elements lie in $\fr{g}^{i-1}$.
Consequently, if $\fr{g}$ is a class $c$ nilpotent Lie algebra, then the Magnus expansion of $\mbf{S}_\fr{g}$ truncates at order $c$ and all Magnus terms $\Omega_{> c}$ vanish. As a special case, abelian $\fr{g}$ is the class $1$ nilpotent Lie algebra, all commutators of $\mbf{S}_\fr{g}$ vanish and we have truncation at order $1$.

Recent progress in continuous-time neural network modeling leverages
rough path theory and path signature \cite{muca2024theoretical, walker2024log, walker2025structured}.
Path signature is deeply related to the Magnus series \cite{chevyrev2024multiplicative,chevyrev2025primer}.
For linear CDE, the path signature can be expressed as
a tensor exponential of the Magnus expansion \cite{friz2022unified}.
In other words, the Magnus expansion is a representation of \emph{log-signature} \cite{morrill2021neural, walker2024log}
when input variation is bounded.
Log-signature underscores the algebraic property of input path,
while the Magnus expansion prioritizes state transition.
The study of the Magnus series is relatively mature and often provides both theoretical and computational benefits
when the path signature admits the Magnus expansion \cite{friz2022unified}.

%% file: A2_Proofs.tex
\section{Proofs}
\label{app:proof}
\subsection{Proof of Lemma~\ref{lem:sim_impossible}}
\label{app:sim_impossible_proof}
\begin{proof}
    Algebraically, it is straightforward.
    Any image of abelian Lie algebra under Lie algebra homomorphism is abelian.
    There is no representation of abelian Lie algebra that depends on path order.
    
    Constructively, general SSM $\mbf{S}$ and abelian SSM $\hbf{S}$ are linear ODEs.
    It suffices to find the existence of nonsingular linear state transformation map $\mbf{T}$
    that satisfies $\mbf{T}\hat{h}(t) = h(t)$ for any time $t$ and admissible inputs
    (Sec.~\ref{par:sim_min_main} and App.~\ref{app:sim_detail}).

    Let us denote $(\Phi, \hat{\Phi})$ be the state-transition matrix of $(\mbf{S}, \hbf{S})$.
    It is clear that $\mbf{T}\hat{h}(t) = h(t)$ implies $\mbf{T}\hat{\Phi}(t,0) = \Phi(t,0)\mbf{T}$.
    Consider an admissible input path $x: [0,T] \to \mcal{X}$,
    and its measure-preserving permutation path $x'(t)= x(\sigma(t))$
    where $\sigma:[0,T] \to [0,T]$.
    Assume $x'$ is also admissible.
    Then, being abelian:
    \begin{equation*}
        \hat{\Phi}_{x}(T,0) = \hat{\Phi}_{x'}(T,0),
        \quad
        \mbf{T}(\hat{\Phi}_{x}-\hat{\Phi}_{x'})\mbf{T}^{-1} = \mbf{0}.
    \end{equation*}
    However, for non-abelian $\mbf{S}$, chronological exponential map depends on the path order and $\Phi_{x}(T,0) \neq \Phi_{x'}(T,0)$,
    contradicting the existence of $\mbf{T}$.
    
    As noted in App.~\ref{app:sim_detail}, the proof generalizes to the case $m = \dim \hbf{S} \ge \dim \mbf{S} = n$.
    $\mbf{T}$ simply needs to satisfy:
    \begin{equation*}
        \mbf{T}\hat{h} = \begin{bmatrix}
            h \\
            *
        \end{bmatrix}
        \begin{matrix*}[l]
            \rbrace \; n \\
            \rbrace \; m-n,
        \end{matrix*}
    \end{equation*}
    for some variables $*$ to match dimension \cite{toth2010modeling}.
    The algebraic nature is intact, and the proof immediately follows.
    With respect to Eq.~\ref{eq:SSM_sim_error},
    truncating $*$-corresponding rows of $\mbf{T}$ recover full-rank matrix $\mbf{P}$.
\end{proof}
\subsection{Proof of Theorem~\ref{thm:sim_error}}
\label{app:sim_error_proof}
Here we divide the problem into abelian and restricted cases.
Notation follows App.~\ref{app:sim_impossible_proof}.
\begin{proof}[Abelian SSM]
    Consider general SSM $\mbf{S}$.
    Assume an input path $x$ and permuted $x'$ on window $[0,T]$ are both admissible.
    With small enough $T$, we can derive a loose lower bound of $\|\Phi_{x}(T) - \Phi_{x'}(T)\|$ with Magnus expansion in App.~\ref{app:CDE}:
    \begin{align*}
        \|\Phi_x(T) - \Phi_{x'}(T)\|
        &= \|\exp \Omega(T) - \exp \Omega'(T) \| \\
        & \ge \exp \Bigl(
        - \max
        \{ \|\Omega(T)\|, \| \Omega'(T) \| \}
        \Bigr)
        \| \Omega(T) - \Omega'(T) \| \\
        &= \lambda\| \Omega(T) - \Omega'(T) \|, \qquad \lambda \in (0,1].
    \end{align*}
    Recall Eq.~\ref{eq:omega_2},
    \begin{align*}
        \|\Omega_2(T) - \Omega_2'(T)\|
        &= 
        \biggl\|
        \frac{1}{4}\iint_{[0,T]^2}
        \Bigl(
        \text{sgn}(t_1 - t_2)
        -
        \text{sgn}(\sigma(t_1) - \sigma(t_2))
        \Bigr)
        [\mbf{A}(x(t_1)), \mbf{A}(x(t_2))]
        \, dt_1 dt_2
        \biggr\| \\
        &\le 
        \biggl \|
        \frac{1}{2} \iint_{[0,T]^2}
        \text{sgn}(t_1 - t_2)
        [\mbf{A}(x(t_1)), \mbf{A}(x(t_2))]
        \, dt_1 dt_2
        \biggr \| \\
        &= 2\|\Omega_2(T)\|,
    \end{align*}
    where equality holds when $\sigma(t)$ is a \emph{time-reversal} function.

    Time-reversal function simplifies Magnus series further.
    It is straightforward that $\Omega(T) - \Omega'(T) = \sum_{i=1} \Omega_{2n}(T)$ with $\Omega_{2n}(T) = -\Omega_{2n}'(T)$ and $\Omega_{2n+1}(T) = \Omega_{2n+1}'(T)$.
    More generally,
    $\Omega_2$ is the leading term in $\Omega_{\ge 2}(T)$
    without assuming that the time-reversal function is a permutation.
    Its tail $\Omega_{\ge 3}(T)$ shrinks at $\mathcal{O}(T^3)$.
    Then, we can derive the rest of the lower bound:
    \begin{equation}
    \label{app:eq_OmegaError_2nd}
        \| \Omega(T) - \Omega'(T) \|
        =
        2\|\Omega_{2n}(T)\|
        \ge 2c\|\Omega_2(T)\|, \qquad c \in (0,1].
    \end{equation}
    On the other hand, App.~\ref{app:sim_impossible_proof} implies abelian $\hbf{S}$ generates the same final state under path permutation.
    Thus, either $x$ or $x'$ should incur state mismatch:
    \begin{equation*}
        \max_{x,x'} \Delta_h
        \ge
        \lambda c \|\Omega_2(T)\|\|h_0\|,
    \end{equation*}
    provides lower bound of approximation error incurred on small window.
    \footnote{For anyone interested in output error,
    the minimality condition guarantees existence of some common suffix sequence
    that generates output mismatch from state mismatch.}
    
    When $T$ is large, we decompose the path by small enough time partitions, as noted in Eq.~\ref{eq:flow_composition}:
    \begin{equation}
        \Phi(T,0) = \Phi(T,t_k)\Phi(t_k,t_{k-1})\cdots\Phi(t_1,0).
    \end{equation}
    Note that Eq.~\ref{eq:flow_composition} admits simple rearrange of terms:
    \begin{align*}
        \Phi_x(t_2,t_1)\Phi_x(t_1,0) &- \Phi_{x'}(t_2,t_1)\Phi_{x'}(t_1,0)
        = \\
        &\Phi_x(t_2,t_1)\bigl(\Phi_x(t_1,0) - \Phi_{x'}(t_1,0)\bigr)
        +
        \bigl(\Phi_x(t_2,t_1) - \Phi_{x'}(t_2,t_1)\bigr)\Phi_{x'}(t_1,0).
    \end{align*}
    Therefore,
    \begin{equation}
        \Phi_x(T) - \Phi_{x'}(T) = \sum_{i=1}^k \Phi_x(T, t_{i+1})
        \Bigl(
        \Phi_x(t_i, t_{i-1}) - \Phi_{x'}(t_i, t_{i-1})
        \Bigr)
        \Phi_{x'}(t_{i-1},0),
    \end{equation}
    and with minimality condition, $x$ can be designed to accumulate local error over each small pieces.
\end{proof}
\begin{proof}[Restricted SSM]
    Let us construct admissible inputs on interval $[0,2T]$ by composing two $T$-length sequence\footnote{
    Choice of $2T$ is arbitrary and two sequences need not have the same length.
    }:
    \begin{equation*}
        x_1 = \text{prefix}_1 \oplus \mbf{xy},
        \quad
        x_1' = \text{prefix}_1 \oplus \mbf{yx},
        \quad
        x_2 = \text{prefix}_2 \oplus \mbf{xy},
        \quad
        x_2' = \text{prefix}_2 \oplus \mbf{yx}
    \end{equation*}
    with permutation map $\sigma$ on $(T, 2T]$.
    Here we abuse notation and denote $(\mbf{xy}, \mbf{yx})$ to emphasize permutation.
    
    Consider general SSM $\mbf{S}$.
    As $x_1$ and $x_1'$ shares prefix, $\Phi_{x_1}(T,0) = \Phi_{x_1'}(T,0)$.
    Shortly, we denote $\mbf{M} := \Phi_{x_1}(2T,T) - \Phi_{x_1'}(2T,T)$.
    By simple subtraction:
    \begin{align*}
        h_1(2T) - h_1'(2T)
        =
        \mbf{M}h_1(T)
        +
        \int_T^{2T}
        \Bigl(
        \Phi_{\mbf{xy}}(2T, \tau)\mbf{b}(\mbf{xy}(\tau))
        -
        \Phi_{\mbf{yx}}(2T, \tau)\mbf{b}(\mbf{yx}(\tau))
        \Bigr) d\tau.
    \end{align*}
    Similarly, we can obtain $h_2(2T) - h_2'(2T)$.
    Then, subtraction between terms cancels inhomogeneous part out:
    \begin{equation}
    \label{app:eq_subtraction_magic}
        \Bigl(
        h_1(2T) - h_1'(2T)
        \Bigr)
        -
        \Bigl(
        h_2(2T) - h_2'(2T)
        \Bigr)
        =
        \mbf{M}
        (h_1(T) - h_2(T)).
    \end{equation}
    On the other hand, for restricted $\hbf{S}$, $\mbf{M} = \mbf{0}$.
    Therefore, among these four inputs,
    \begin{align}
    \label{app:eq_inhomogeneous_error}
        \max_{x_1,x_1', x_2, x_2'} \Delta_h
        &\ge \frac{1}{4} \|\mbf{M}
        (h_1(T) - h_2(T))\|,
    \end{align}
    with magnitude of $\mbf{M}$ following Eq.~\ref{app:eq_OmegaError_2nd}.
    
    With minimality condition,
    $\text{prefix}_1$ and $\text{prefix}_2$ can be chosen such that $h_1(T) \neq h_2(T)$ for $\mbf{S}$.
    Up to bounded magnitude, we can make $h_1(T)$ and $h_2(T)$ be as far as possible.
    Long sequence behavior is analogous to the abelian case.
\end{proof}
\subsection{Proof of Lemma~\ref{lem:stackderivedlength}}
\label{app:stackderivedlength_proof}
We first show abelian $k$-layer SSM has up to $k$ derived length,
and naturally derives $2k$ for restricted SSMs.
To obtain notational simplicity, we choose vector field notation for this proof.
\begin{proof}[Abelian SSM]
    Consider abelian $k$-layer SSM with vector field
    $F(t,h,x) = (F_0, F_1, \ldots, F_{k-1})$ on a manifold $\mcal{M} \cong \bb{R}^n, \mcal{M}_i \cong \bb{R}^{n_i}$.
    Assume $(t,x)$ is fixed.
    For the layer index $i$ and state coordinate index $a$, we trivially extend the local vector field:
    \begin{equation}
    \label{app:eq_trivial_support_extension}
        h_i \in \mcal{M}_i,
        \quad
        \partial_{i,a} = \frac{\partial}{\partial h_i^a},
        \quad 
        \bar{F}_i(h) = \sum_{a=1}^{n_i} F_i^a (h_{\le i}) \partial_{i,a}
        = (\underbrace{0,\ldots,0}_{0:i-1}, F_i, \underbrace{0,\ldots,0}_{i+1:k-1}),
    \end{equation}
    such that the coefficient function $(\bar{F}_i)_j^c = 0$ for any $j \neq i$ and $\bar{F}_i^c = F_i^c$.
    Then the global vector field $F(h) = \sum_{i=0}^{k-1} \bar{F}_i(h)$.
    
    Now consider two fixed control inputs $x,y$ and corresponding vector fields $X,Y$
    such that $X(h) = F(h,x)$ and $Y(h) = F(h,y)$.
    Following Eq.~\ref{app:eq_trivial_support_extension}, let us denote:
    \begin{align*}
        X = \sum_{i=0}^{k-1} X_i,
        \quad
        X_i &= \sum_{a=1}^{n_i}X_i^a \partial_{i,a},
    \end{align*}
    and similarly for $Y$.

    With Eq.~\ref{app:eq_LieBracket_Directional} let us construct their Lie bracket $[X,Y] = \sum_{i,j}[X_i,Y_j]$.
    When $i=j$, trivially $[X_i,Y_i] = 0$ from abelian definition.
    For $i < j$, note that $X_i$ is independent of $h_j$ by construction, therefore:
    \begin{align}
        [X_i,Y_j]_i^c &= \sum_{a=1}^{n_i}X_i^a \partial_{i,a} (Y_j)_i^c
        - \sum_{b=1}^{n_j}Y_j^b \partial_{j,b} X_i^c = 0, \label{eq:zerocrossbracket} \\
        [X_i,Y_j]_j^c &= \sum_{a=1}^{n_i}X_i^a \partial_{i,a} Y_j^c
        - \sum_{b=1}^{n_j}Y_j^b \partial_{j,b} (X_i)_j^c
        = \sum_{a=1}^{n_i}X_i^a \partial_{i,a} Y_j^c, \label{eq:nonzerocrossbracket}
    \end{align}
    such that nonzero component of $[X_i,Y_j]$ only lives in the $j$-th layer.

    Eq.~\ref{eq:nonzerocrossbracket} provides a chain of ideals.
    Let us denote $\fr{g}$ be the Lie algebra generated by the collection of global vector fields $F$.
    We define its ideal $I_{l} := \{X \in \fr{g} \mid X(h) = \sum_{i \ge l} X_i(h)\}$ with zero component below layer $l$.
    Then:
    \begin{equation*}
        [X_i,Y_j] \in I_{\max (i,j)},
        \quad
        [X,Y] \in I_1,
        \quad
        \fr{g}^{(n)} \in I_n,
    \end{equation*}
    which allows us to construct a chain,
    \begin{equation}
    \label{eq:cascade_derived_series}
        \fr{g} = I_0 \supseteq I_1 \supseteq \cdots \supseteq I_{k-1} \supseteq I_k = \{0\},
    \end{equation}
    where quotients of ideals, $I_i / I_{i+1}$, are all abelian by construction.
    Eq.~\ref{eq:cascade_derived_series} is not necessarily the shortest sequence of abelian ideals,
    so abelian $k$-layer SSM has maximally $k$ derived length.
\end{proof}
\begin{proof}[Restricted SSM]
    It is trivial from Eq.~\ref{eq:cascade_derived_series}.
    Each $I_i / I_{i+1}$ is constructed to be restricted,
    which by themselves are solvable with derived length $2$.
    Again, the series is not necessarily the shortest derived series,
    and restricted $k$-layer SSM has derived length at most $2k$.
\end{proof}

\subsection{Proof of Theorem~\ref{thm:depth_extension}}
\label{app:sub_proof_depth_extension}
\begin{proof}
    Consider lifted Lie equation of $\mbf{S}_\fr{g}$ with a matrix Lie algebra $\fr{g} \subset \fr{gl}(V)$:
    \begin{equation}
    \label{app:eq_stack_Lie_eq}
        \dot{\mbf{G}}(t) = \mbf{A}_\fr{g}(x(t))\mbf{G}(t),
        \quad
        \mbf{G}(0) = I_V,
    \end{equation}
    where $\mbf{A}_{\fr{g}}(x(t)) \in \fr{g}$. Given the Lie algebra extension:
    \begin{equation}
    \label{app:eq_proofs_Lie_alg_extension_seq}
        0 \to \fr{a} \xrightarrow{i} \fr{g} \xrightarrow{\theta} \fr{h} \to 0,
    \end{equation}
    our interest is to decompose the integral curve $\mbf{G}(t)$ into curves integrating ideal $\fr{a}$ and quotient $\fr{h} \cong \fr{g} / \fr{a}$.
    
    Let $G$ be a matrix Lie group locally integrating $\fr{g}$ near neighborhood of identity. Again locally, let $A$ be the normal subgroup integrating the ideal $\fr{a}$, and $H$ for the quotient Lie algebra $\fr{h}$.
    We do not discuss global property of $(A, G, H)$, and shortly call them as local matrix Lie groups. 
    Restricted to sufficiently small neighborhoods,
    \begin{equation*}
        1 \to A \xrightarrow{\boldsymbol i} G \xrightarrow{\Theta} H \to 1,
    \end{equation*}
    where differentiation at the identity recovers Eq.~\ref{app:eq_proofs_Lie_alg_extension_seq}.
    We follow common notational convention and do not distinguish
    $A$ or $\fr{a}$ from $\boldsymbol i(A)$ or $i(\fr{a})$,
    respectively.
    
    For the quotient part, its Lie equation is straightforward.
    We often silence input variables without confusion for the clarify.
    With $\theta = d\Theta_e$:
    \begin{align}
        \mbf{H} &:= \Theta(\mbf{G}) \in H,
        \quad
        \bar{\mbf{A}} := \theta(\mbf{A}_\fr{g}) \in \fr{h}, \nonumber
        \\
        \dot{\mbf{H}}
        &=
        \frac{d}{dt}\Theta(\mbf{G})
        =
        (d\Theta_{\mbf{G}})(\dot{\mbf{G}})
        =
        \theta(\mbf{A}_\fr{g})\Theta(\mbf{G})
        =
        \bar{\mbf{A}}\mbf{H}. \label{app:eq_quotient_LieEquation}
    \end{align}
    Now let us construct the Lie equation generated by ideal $\fr{a}$. We take a smooth local section,
    \begin{equation*}
        \mbf{s}: H \to G,
        \quad
        \Theta(\mbf{s}(\mbf{H}(t))) = \mbf{H}(t).
    \end{equation*}
    As $\ker \Theta = \text{im} \: \boldsymbol i$ from the definition of the short exact sequence,
    we can characterize the ideal.
    Given an inversion map on group $\iota: G \to G$,
    \begin{equation}
    \label{app:eq_ideal_LieEquation}
        \mbf{a}(t) := \mbf{Q}^{-1}(t)\mbf{G}(t),
        \quad
        \mbf{Q}(t) := \mbf{s}(\mbf{H}(t)) \in G,
        \quad
        \mbf{Q}^{-1}(t) = \iota(\mbf{Q}(t)).
    \end{equation}
    We can confirm $\mbf{a}(t)$ is indeed an element of $A$, sitting in $\ker \Theta$:
    \begin{equation*}
        \Theta(\mbf{Q}^{-1}\mbf{G}) = \Theta(\mbf{Q}^{-1})\Theta(\mbf{G}) = \Theta(\iota(\mbf{s}(\mbf{H})))\mbf{H} = \mbf{H}^{-1}\mbf{H} = e_{H},
    \end{equation*}
    therefore any $\mbf{a}(t) \in A$.
    Differentiation gives us the Lie equation of $\fr{a}$.
    \begin{equation}
    \label{app:eq_ideal_first_differentiation}
        \dot{\mbf{a}} = \left(\frac{d}{dt} \mbf{Q}^{-1} \right)\mbf{G} + \mbf{Q}^{-1}\dot{\mbf{G}}
        = -\mbf{Q}^{-1}\dot{\mbf{Q}}\mbf{Q}^{-1}\mbf{G} + \mbf{Q}^{-1}\dot{\mbf{G}}
        = (\mbf{Q}^{-1}\mbf{A}_\fr{g}\mbf{Q}-\mbf{Q}^{-1}\dot{\mbf{Q}})\mbf{a},
    \end{equation}
    where $\dot{\mbf{Q}}(t)$ follows,
    \begin{equation*}
        \dot{\mbf{Q}} = (d\mbf{s})_{\mbf{H}}(\dot{\mbf{H}}) = (d\mbf{s})_{\mbf{H}}(\bar{\mbf{A}}\mbf{H}).
    \end{equation*}
    Again let us confirm $(\mbf{Q}^{-1}\mbf{A}_\fr{g}\mbf{Q}-\mbf{Q}^{-1}\dot{\mbf{Q}}) \in \fr{a}$
    to ensure Eq.~\ref{app:eq_ideal_first_differentiation} is indeed the Lie equation of $\fr{a}$.
    
    First, note that $\mbf{Q}^{-1}\mbf{A}_\fr{g}\mbf{Q}$ is an adjoint action (or conjugation) of Lie group on its Lie algebra
    \cite{robbin2022introduction}.
    Being matrix Lie group and algebra, it is a standard matrix calculus:
    \begin{equation*}
        \theta(\mbf{Q}^{-1}\mbf{A}_\fr{g}\mbf{Q})
        =
        \theta(\text{Ad}_{\mbf{Q}^{-1}}\mbf{A}_\fr{g})
        =
        \text{Ad}_{\Theta(\mbf{Q}^{-1})}\bigl(\theta(\mbf{A}_\fr{g})\bigr)
        =
        \Theta(\mbf{Q}^{-1})\theta(\mbf{A}_\fr{g})\Theta(\mbf{Q})
        =
        \mbf{H}^{-1}\bar{\mbf{A}}\mbf{H}.
    \end{equation*}
    The second term $\mbf{Q}^{-1}\dot{\mbf{Q}}$ is simply derived from Eq.~\ref{app:eq_quotient_LieEquation}:    
    \begin{equation*}
        \theta(\mbf{Q}^{-1}\dot{\mbf{Q}}) = \mbf{H}^{-1}\dot{\mbf{H}} = \mbf{H}^{-1}\bar{\mbf{A}}\mbf{H},
    \end{equation*}
    then terms cancel out,
    \begin{equation*}
        \theta(\mbf{Q}^{-1}\mbf{A}_\fr{g}\mbf{Q}-\mbf{Q}^{-1}\dot{\mbf{Q}}) = 0.
    \end{equation*}
    Recall from Eq.~\ref{app:eq_proofs_Lie_alg_extension_seq} that $\ker \theta = \text{im} \; i$.
    Therefore, Eq.~\ref{app:eq_ideal_first_differentiation} is a Lie equation generated by the ideal.
    
    Together, Eq.~\ref{app:eq_stack_Lie_eq} locally admits 2-layer decomposition:
    \begin{align*}
        \mbf{G}(t) = \mbf{s}(\mbf{H}(t))\mbf{a}(t),
        \quad
        \dot{\mbf{H}}(t) &= \bar{\mbf{A}}(x(t))\mbf{H}(t),
        \quad
        \dot{\mbf{a}}(t) = \mbf{\bar{Q}}(x(t), \mbf{H}(t))\mbf{a}(t), \\
        \mbf{\bar{Q}}(x(t), \mbf{H}(t)) &= \mbf{Q}^{-1}(t)\mbf{A}_\fr{g}(x(t))\mbf{Q}(t)-\mbf{Q}^{-1}(t)\dot{\mbf{Q}}(t).
    \end{align*}
    Therefore, the lifted Lie equation of $\mbf{S}_\fr{g}$ decomposes into 2 Lie equations where each represents ideal $\fr{a}$ and quotient $\fr{h}$.
    With Eq.~\ref{eq:flow_to_state_submersion}, decomposed Lie equations
    can locally simulate $\mbf{S}_\fr{g}$.
    To project Lie equation back to SSM,
    we vectorize with the Kronecker product $\otimes$.
    For example, in Eq.~\ref{app:eq_stack_Lie_eq},
    consider $\fr{g} \subset \fr{gl}(n,\bb{R})$.
    Let $\mbf{z} \in \bb{R}^{n^2}$ be a flattened $\mbf{G} \in \bb{R}^{n \times n}$:
    \begin{equation}
    \label{app:eq_Lie_to_SSM_vectorization}
        \dot{\mbf{z}}(t)
        =
        \bigl(
        I_n \otimes \mbf{A}_\fr{g}(x(t))
        \bigr)
        \mbf{z}(t),
        \quad
        \mbf{z}(0) = \text{vec}(I_n).
    \end{equation}
    Therefore, $\mbf{S}_\fr{g}$ can be decomposed and simulated by 2-layer deep SSM.
\end{proof}
\subsection{Proof of Corollary~\ref{col:Kstacks}}
\label{app:sub_proofKstacks}
\begin{proof}
    Recall that a solvable Lie algebra $\fr{g}$ with derived length $k$ satisfies following:
    \begin{equation}
    \label{app:eq_proofs_solvable_algebra_seq}
        \fr{g} = \fr{g}^{(0)} \supseteq \cdots \supseteq \fr{g}^{(k-1)} \supseteq \fr{g}^{(k)} = \{0\},
        \quad
        0 \to \fr{g}^{(k-1)} \xrightarrow{i} \fr{g}  \xrightarrow{\theta} \fr{g} / \fr{g}^{(k-1)} \to 0.
    \end{equation}
    
    We invoke Thm.~\ref{thm:depth_extension} to decompose $\mbf{S}_\fr{g}$ by replacing $\fr{a}$ with $\fr{g}^{(k-1)}$ and $\fr{h}$ with $\fr{g}/\fr{g}^{(k-1)}$ in Eq.~\ref{app:eq_proofs_Lie_alg_extension_seq}.
    As the base (quotient) layer $\fr{g}/\fr{g}^{(k-1)}$ is again solvable with derived length $k-1$,
    we can again faithfully represent $\fr{g} / \fr{g}^{(k-1)}$ as a matrix Lie algebra and iteratively apply the same decomposition near identity.

    Therefore, a solvable $\mbf{S}_\fr{g}$ can be decomposed and simulated by an abelian deep SSM.
\end{proof}
\subsection{Proof of Corollary~\ref{col:nilpotentization}}
\label{app:sub_proof_nilpotentization}
The flow of non-solvable system $\mbf{S}_\fr{g}$ demands infinite nontrivial terms in Magnus expansion.
Nilpotentization, or truncation, of the Magnus expansion is
an approximation technique that matches up to $c$-th order Magnus term \cite{iserles2008magnus,blanes2009magnus}.

As noted in Sec.~\ref{app:CDE}, the $c$-th order Magnus term involves $(c-1)$ Lie bracket operations. In lower central series notion, the term involves $\fr{g}^{c-1}$.
The idea is to quotient out the higher-order brackets and construct a class $c$ nilpotent Lie algebra.
This allows us to keep Magnus series intact up to $\Omega_c$, while any $\Omega_{>c}$ vanishes.

More explicitly, for any Lie algebra $\fr{h}$, the derived series and lower central series satisfy:
\begin{equation*}
    \fr{h}^{(k)} \subseteq \fr{h}^{2^k-1}.
\end{equation*}
Therefore, if $\fr{h}$ is nilpotent of class $2^k-1$,
its derived length is upper bounded by $k$ \cite{burde2012derived}.
In other words, if $\fr{h}$ is nilpotent of class $c$,
its derived length $k$ is upper bounded by:
\begin{equation*}
    k \le \lceil \log_2 (c+1) \rceil = \lfloor \log_2 c \rfloor + 1,
\end{equation*}
as $c$ is integer.
Combined with Cor.~\ref{col:Kstacks},
abelian $k$-layer SSM suffices to express the truncated Magnus expansion.

Following, the leading term of Magnus series truncation error is $\|\Omega_{2^k}\| \lesssim \mcal{O}(\epsilon^{2^k})$ \cite{blanes2009magnus}.
Therefore, adding layers can exponentially increase the truncation order
and decrease the magnitude of leading error term.
\subsection{Proof of Proposition~\ref{prop:logdepth}}
\label{app:sub_proof_logdepth}
Here we begin with a key concept for the proof.
\begin{definition}[Lyndon words \cite{avdieiev2024affine}]
\label{app:def_LyndonWords}
    Let $\mcal{A}$ be totally ordered.
    For any words, we introduce the \emph{lexicographical} order on $\mcal{A}^*$:
    \begin{equation}
    a_1a_2\ldots a_i < b_1b_2\ldots b_j \quad \text{if} \begin{cases}
            a_1 = b_1, a_2 = b_2, \ldots, a_n = b_n, a_{n+1} < b_{n+1}\text{ for some } n \ge 0 \\
            \text{or} \\
            a_1 = b_1, \ldots, a_i = b_i \text{ and } i < j.
        \end{cases}
    \end{equation}
    Then, a word $w = a_1a_2\ldots a_i$ is \emph{Lyndon} if it is smaller than all of its cyclic permutations:
    \begin{equation*}
        a_1\ldots a_{n-1}a_n \ldots a_i < a_n\ldots a_i a_1\ldots a_{n-1}
        \quad
        \forall n \in \{2, \ldots, i\}.
    \end{equation*}
    and we call $a_1\ldots a_{n-1}$ and $a_n \ldots a_i$ as \emph{prefix} and \emph{suffix} of the word, respectively.
\end{definition}
\begin{proof}
With the definition of word problems in Def.~\ref{def:wp_main_text}, 
the Chen-Fox-Lyndon theorem states that any words have a unique \emph{canonical} factorization in Lyndon words \cite{chen1958free,melanccon1989lyndon,avdieiev2024affine}.
Let us denote a set of Lyndon words on $\mcal{A}$ as $L$, and its subset up to length $T$ as $L_T$.
For any word $w$,
\begin{equation*}
    w = l_1l_2\ldots l_k,
    \quad
    l_1 \ge l_2 \cdots \ge l_k,
    \quad
    l_i \in L.
\end{equation*}
Notice that the canonical factorization of a Lyndon word is trivially itself.
Instead, Lyndon word has a \emph{costandard} factorization. For any Lyndon word $l$,
\begin{equation*}
    l = l_1l_2,
    \quad
    l_1, l_2 \in L,
\end{equation*}
where $l_2$ is the longest proper Lyndon suffix of $l$.

Clearly, the canonical factorization provides the key insight in our proof.
Intuitively, when word length is bounded at $T$,
it suffices to "memorize" Lyndon words and their composition up to length $T$ to solve the problem.

It is known that $L$ provides a basis of \emph{free} Lie algebra on $\mcal{A}$ \cite{avdieiev2024affine, walker2024log}.
Let $\fr{L}$ be a free Lie algebra generated by a finite set $\{e_a\}_{a \in \mcal{A}}$,
where each element is labeled by an alphabet.
The bracket rule is given:
\begin{equation}
    \text{b}[a] = e_a \in \fr{L} \text{ for } a \in \mcal{A},
    \quad
    \text{b}[l] = [\text{b}[l_1],\text{b}[l_2]],
\end{equation}
with $l=l_1l_2$ be the costandard factorization.
For a simple example, consider a Lyndon word $a_1a_2a_3$ with costandard factorization $(a_1, a_2a_3)$. Then, 
\begin{equation*}
    \text{b}[a_1a_2a_3] = [\text{b}[a_1], \text{b}[a_2a_3]] = [e_{a_1}, [\text{b}[a_2], \text{b}[a_3]]] = [e_{a_1}, [e_{a_2}, e_{a_3}]] \in \fr{L}.
\end{equation*}
Observe that Lyndon word length $3$ corresponds to the second-order Lie bracket.
Analogously, Lyndon word length $T$ corresponds to $(T-1)$ Lie bracket operations.

Therefore, truncating the Lyndon words at $L_T$ implies that
we retain the lower central series of $\fr{L}$ up to $(T-1)$-th order, or $\fr{L}^{T-1}$.
The first discarded lower central series is $\fr{L}^T$, which involves $T$ Lie bracket operations and thus Lyndon word longer than $T$.
Effectively, we take a quotient of $\fr{L}$ with $\fr{L}^T$.
As $\fr{L}^T$ is an ideal in $\fr{L}$, the quotient is also a Lie algebra.
Following, $\fr{g} = \fr{L} / \fr{L}^T$ is a finite-dimensional nilpotent Lie algebra with class at most $T$.

Now we invoke Cor.~\ref{col:Kstacks} just as we did in App.~\ref{app:sub_proof_nilpotentization}.
For a class-$T$ nilpotent $\mbf{S}_\fr{g}$, its derived length is at most $\lfloor \log_2 T \rfloor + 1$.
Therefore, we can construct abelian $\bigl(\lfloor \log_2T \rfloor + 1\bigr)$-layer SSM plus output layer to simulate flow \cite{burde2012derived}.
Each alphabet symbol is then a piecewise-constant external input signal to $\mbf{S}_\fr{g}$.
\end{proof}
\subsection{Proof of Corollary~\ref{col:logdepth_state}}
\label{app:sub_proof_logdepth_state}
The construction is natural from Prop.~\ref{prop:logdepth}.
Following logics and notations in App.~\ref{app:sub_proof_logdepth},
we highlight that the number of Lyndon words determines the dimension of $\fr{g}$.
Let the number of alphabet symbols $|\mcal{A}| = n$,
and denote $\fr{g}(T,n) = \fr{L} / \fr{L}^T$.
Then, \emph{Witt's formula} provides following \cite{magnus1966combinatorial, bouallegue2010primitive}:
\begin{equation*}
    \dim \fr{g}(T,n) = \sum_{m=1}^T \frac{1}{m} \sum_{d \mid m} n^d \ \mu\bigl(\frac{m}{d} \bigr),
\end{equation*}
where $\mu$ is the M\"obius function.

%% file: A3_Experiments.tex
\section{Experimental Details}
\label{app:exp}
Here we provide further experimental details.
Upon acceptance, we will release all our code used to produce our experimental results on a public GitHub repository.
We conducted all our experiments using a single GPU (A100 or H100) and no run lasted longer than three hours.

\subsection{Word problems}
\label{app:exp_wp}
\paragraph{Task formulation.}
We encode each of the word problems used in Table \ref{tab:word_problem} as a causal sequence labeling problem.
Each word problem has a finite set of group elements.
A task sequence is constructed by assigning a unique index to each group element in the corresponding word problem.
Given a sequence of indexed transformations, the task is to predict, at each position in the sequence, the class label corresponding to the group element obtained by composing the transformations in sequence.
Training and test data sets are constructed by randomly sampling such sequences.

In practice, the corresponding word problem sequence is prepended with an artificial
\emph{beginning} of sequence (BOS) token.
Conventionally, we assign the same BOS token as the target.
BOS token allows the model to \emph{burn in} an initial state before processing the actual problem sequence.

In our preliminary studies, we also tried to explicitly learn the initial states of the sequence models (for GLA and Mamba variants) as part of the trainable parameters.
As we did not find this to be helpful, all our experiments are conducted without learning the initial states.
Modes are trained to make prediction at each position in the sequence, up to the maximum length of the data.

A sequence-level prediction is only counted as correct if the model successfully predicts all the elements in the sequence correctly.

\paragraph{Model configurations.}
Models are separately trained on each word problem. Given that each word problem has a different number of group element, their base vocabulary size is different for each case (namely, 2 for parity/$C_2$, 3 for $C_3$, 8 for $D_8$ and $H_3$, 6 for $S_3$, 24 for $S_4$, and 60 for $A_5$).
In addition to these base vocabulary, additional special tokens such as the padding token, including BOS, are allocated.
Such a choice is conventional for most modern sequence model implementations, which we follow;
none of them is required in our tasks except BOS.

The numbers reported in Table \ref{tab:word_problem} are the best results from
training all hyper-parameters configuration listed in Table \ref{tab:hyperparameters}
for three random seeds each for 100 epochs.
We monitor length generalization performance, and stop early
if the model achieves 100\% sequence-level prediction accuracy.
This corresponds to the standard evaluation protocol on word problems \citep{merrill2024illusion,siems2025deltaproduct}, and serves as an operational definition of a model that is unsuccessful at performing a certain task.
For all tasks except the simplest $C_2$, we use 500k training sequences, and evaluation is conducted on 1000 test sequences. 
As a technical side note, Deltaproduct achieved 100\% generalization accuracy with only 100K training sequences for all the tasks; nevertheless, to potentially help other models, we increased it to 500K.

We use sequence length of 128 for training and 256 for testing.

\paragraph{Implementation.}
We use \texttt{flash-linear-attention} \citep{yang2024fla} to implement transformer, GLA, and DeltaProduct models.
We use public code from the respective papers for signed Mamba \citep{grazzi2024unlocking} and AUSSM \citep{karuvally2025bridging}.

Instead of signed Mamba, one could also use complex or signed diagonals for GLAs in theory.
However, GLA's efficient implementations of the chunk-wise training algorithm assume positive diagonals as also noted by \citet{grazzi2024unlocking}.
Therefore, we focus on signed Mamba only.

Note that the AUSSM implementation requires us to disable mixed precision for compatibility with the use of complex numbers.
Furthermore, the AUSSM backward kernel assumes sequence length to be a multiple of 8,
requiring some care about the length of sequences we feed in practice.

\begin{table*}[t]
\caption{
Hyper-parameter search space. The number of heads is set to 8 for transformer variants, and $d\_\text{state}$ (i.e., the ``head dimension'') is set to 16 for SSMs. Each configuration is run with 3 random seeds.
}
\label{tab:hyperparameters}
\begin{center}
\begin{tabular}{rc}
\toprule
Parameters & Values  \\ \midrule
Hidden size & \{64, 128, 256\} \\
Learning rate & \{1e-2, 1e-3, 3e-4, 1e-4\} \\
Batch size & \{256, 2048\} \\
\bottomrule
\end{tabular}
\end{center}
\end{table*}

\paragraph{Model class.}
AUSSM is complex-valued, but this does not leave our SSM model class.
$n$-dimensional complex diagonal SSMs are equivalent to $m$-dimensional real block-diagonal SSMs
\cite{orvieto2023resurrecting}.
Algebraically, \emph{complexifying} Lie algebras on real field is a standard practice.
For example, any real solvable finite-dimensional Lie algebra complexifies to complex solvable Lie algebra $\fr{g}$.
This provides additional clarity as Lie's theorem allows us to represent any $\fr{g}$ as an upper triangular matrix
over $\bb{C}$ \cite{hall2013lie}.

The transformer is included as a reference model that is parallelizable and order-symmetric with respect to the layer input.
We also leverage interpretation in \citet{choromanski2020rethinking,peng2021random} and \citet{sieber2024understanding}.
Note that the softmax self-attention is algebraically order-symmetric and thus abelian.

It is straightforward to relate GLA, (Signed) Mamba, and DeltaProduct with our SSM model class.

\subsection{$A_5$-rotated vector prediction task}
\label{app:exp_rotation}
Here we provide further details of the rotation task in Sec.~\ref{sec:exp_rotation} and Fig.~\ref{fig:rotation_mse}.
Note that a task description is as provided in the main text: the target vectors are generated by using the matrices provided in the main text.
All the 60 elements of $A_5$ task can be expressed using $\{\mbf{P}, \mbf{R} | \mbf{P}^3=\mbf{R}^5=(\mbf{PR})^2 =\mbf{I}\}$.

We also conducted experiments with AUSSM.
However, we did not report them in the main text as we obtained generally subpar performance.
We conjecture this as the practical training issue.
In particular, complex-valued models are notoriously known to have practical trainability issue (see, e.g., \citet{elelimy0BW24}).
For the sake of completeness, we report the corresponding results here in Fig.~\ref{fig:rotation_mse_aussm}.

\paragraph{Hyper-parameters.} All models are trained using a batch size of 256, learning rate of $1e^{-3}$, with a hidden size of 128.
The number of layers is varied from 1 to 8.
All the models are trained on 500k training sequences of length 128 for 50 epochs.
For deeper models, we have investigated larger hyper-parameter space following Table \ref{tab:hyperparameters}, but we did not obtain any better results.

\begin{figure*}[t]
    \begin{center}
        \includegraphics[width=0.95\columnwidth]{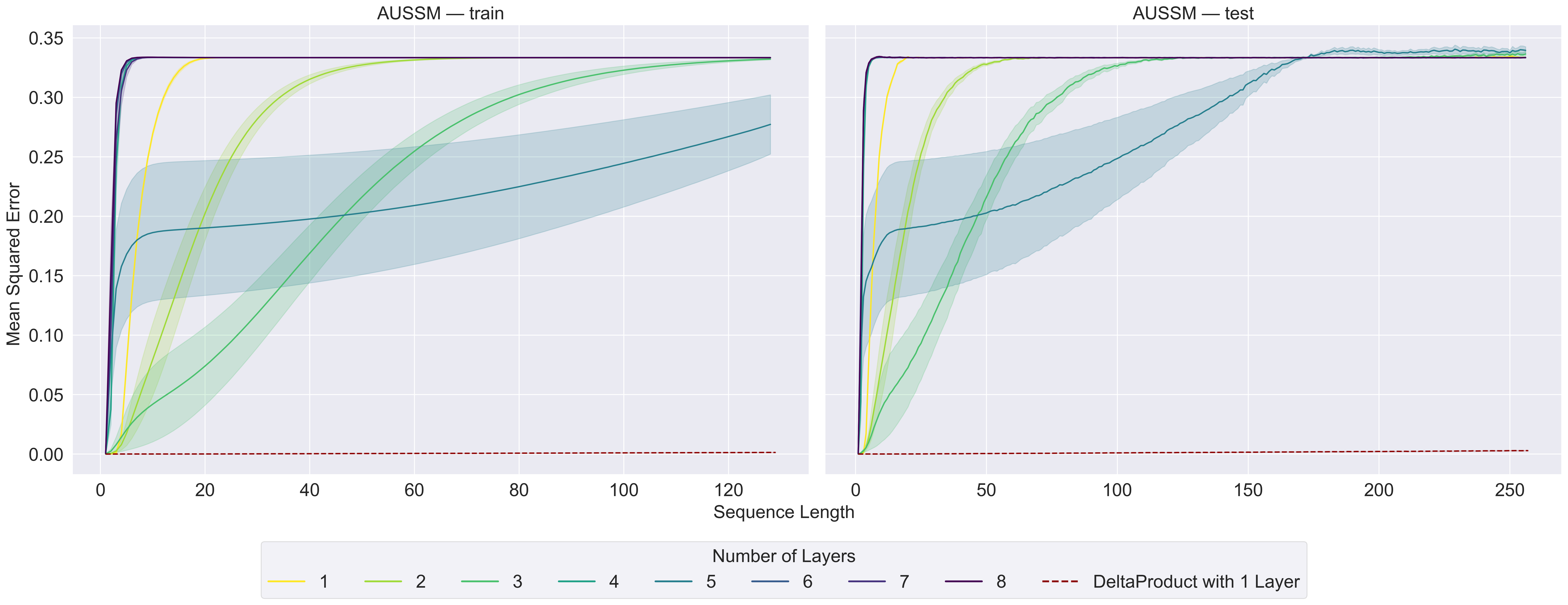}
        \caption{Mean squared loss on the rotated vector prediction task at various sequence length for AUSSM, with different numbers of layers, on train (left) and test (right) sets. Standard error is shown using 3 seeds. Performance of DeltaProduct with 4 Householder products is shown as a reference.
}
        \label{fig:rotation_mse_aussm}
    \end{center}
\end{figure*}